\documentclass[onefignum,onetabnum]{siamonline220329}

\usepackage{lipsum}
\usepackage{amsfonts}
\usepackage{graphicx}
\usepackage{epstopdf}
\usepackage{algorithmic}
\usepackage{algorithm}
\usepackage{subcaption}
\ifpdf
  \DeclareGraphicsExtensions{.eps,.pdf,.png,.jpg}
\else
  \DeclareGraphicsExtensions{.eps}
\fi

\usepackage{enumitem}
\setlist[enumerate]{leftmargin=.5in}
\setlist[itemize]{leftmargin=.5in}

\newsiamremark{remark}{Remark}
\newsiamremark{hypothesis}{Hypothesis}
\crefname{hypothesis}{Hypothesis}{Hypotheses}
\newsiamthm{claim}{Claim}

\headers{CA-PCA: Manifold Dimension Estimation, Adapted for Curvature}{A. Gilbert and K. O'Neill}

\title{CA-PCA: Manifold Dimension Estimation, Adapted for Curvature\thanks{Submitted to the editors May 25, 2023.}
\funding{}}

\author{Anna C. Gilbert\thanks{ Applied Mathematics Program, Yale University, New Haven, CT and Flatiron Institute, New York, NY
  (\email{anna.gilbert@yale.edu}).}
\and Kevin O'Neill\thanks{ The MITRE Corporation
(\email{koneill@mitre.org})}}

\usepackage{amsopn}

\newcommand{\R}{\mathbb{R}}

\newcommand{\tr}{\text{Tr}}

\DeclareMathOperator*{\argmin}{argmin}

\ifpdf
\hypersetup{
  pdftitle={CA-PCA: Manifold Dimension Estimation, Adapted for Curvature},
  pdfauthor={Anna Gilbert and Kevin O'Neill}
}
\fi

\begin{document}

\maketitle

\begin{abstract}
The success of algorithms in the analysis of high-dimensional data is often attributed to the manifold hypothesis, which supposes that this data lie on or near a manifold of much lower dimension. It is often useful to determine or estimate the dimension of this manifold before performing dimension reduction, for instance. Existing methods for dimension estimation are calibrated using a flat unit ball. In this paper, we develop CA-PCA, a version of local PCA based instead on a calibration of a quadratic embedding, acknowledging the curvature of the underlying manifold. Numerous careful experiments show that this adaptation improves the estimator in a wide range of settings.
\end{abstract}



\section{Introduction}

Much of modern data analysis in high dimensions relies on the premise that data, while embedded in a high-dimensional space, lie on or near a submanifold of lower dimension. This allows one to embed the data in a space of lower dimension while preserving much of the essential structure, with benefits including faster computation and data visualization.

This lower dimension, hereafter referred to as the intrinsic dimension (ID) of the underlying manifold, often enters as a parameter of the dimension-reduction scheme. For instance, in each of the Johnson-Lindenstrauss-type results for manifolds by \cite{clarkson2008tighter} and \cite{baraniuk2009random} the target dimension depends on the ID. Furthermore, the ID is a parameter of popular dimension reduction methods such as t-SNE \cite{van2008visualizing} and multidimensional scaling \cite{chen2008multidimensional,france2010two}. Therefore, it may be beneficial to estimate the ID before running further analysis since compressing the data too much may destroy underlying structure and it may be computationally expensive to re-run algorithms with a new dimension parameter, if such an error is even detectable. \cite{block2022intrinsic} uses their ID estimator to obtain sample complexity bounds for Generative Adversarial Networks (GANs). An interesting direct use of ID is found in \cite{santos2022role}, in which it is interpreted as a measure of complexity or variability for basketball plays. To similar effect, \cite{altan2021estimating} uses intrinsic dimension to measure the complexity of the space of observed neural patterns in the human brain.

The literature on ID estimators is vast, and we refer the reader to \cite{camastra2016intrinsic} and \cite{campadelli2015intrinsic} for a comprehensive review. We highlight recent progress of \cite{MiraEtAl}, as well as \cite{grillet2022effective}, \cite{block2022intrinsic}, and \cite{lim2021tangent}, which prove results about the number of samples needed to estimate the ID with a given probability.

For our current analysis, we focus on the use of principal component analysis (PCA) as an ID estimator. There are a variety of ways this has been done (see \cite{bouveyron2011intrinsic,fukunaga1971algorithm,guan2009sparse,little2009estimation,verveer1995evaluation}), but each version of local PCA works roughly by finding the eigenvalues of the covariance matrix for a data point and its nearest neighbors. The ID is then estimated using the fact that we expect the eigenvalues to be high for eigenvectors which are close to the tangent space of the underlying manifold and the eigenvalues to be low for eigenvectors which are nearly orthogonal to the tangent space. We expect $d$ eigenvalues to be larger than the rest, where $d$ is the ID. The remaining $D-d$ eigenvalues will be small, but often nonzero due to effects of curvature or error in measurement.

Local PCA and other ID estimators are often calibrated via a subspace or unit ball intersected with a subspace. While a manifold may be locally well-approximated by its tangent space near a point, in practice one may not always have enough sampled data to zoom in sufficiently close. \textit{Our key insight is that if one expects data to lie on a manifold with nontrivial curvature rather than a subspace, we should calibrate the ID estimator using a manifold with nontrivial curvature rather than a subspace. In particular, our main contribution is to calibrate a version of local PCA to a quadratic embedding, producing a new ID estimator, which we deem curvature-adjusted PCA (CA-PCA).} Provided the underlying manifold is $C^2$-smooth, it will be well-approximated by a quadratic embedding, in fact, better approximated than by its tangent plane.

In principle, this insight could be applied to any number of existing ID estimators. We choose to apply it to a single estimator for which the adjustment is relatively simple, then test the new version extensively. In particular, we adapt a version of local PCA found in \cite{lim2021tangent} for the curvature of a manifold and apply the new version to various examples of data sampled from manifolds. We note that this adaption of local PCA necessitates an entirely novel analysis.

The main benefit of our estimator is that we are better able to estimate the ID in cases where the sample size is small. It allows one to consider either a larger number of nearest neighbors for the ID estimation as it adjusts for the neighborhood starting to ``go around'' the manifold or a smaller number of nearest neighbors with the increased power to distinguish between whether the variance in eigenvalues is due to curvature or statistical noise. Our method achieves this increase in accuracy by regularizing fit of the eigenvectors of the data to a curvature adjusted benchmark. 

The use of quadratic embeddings for data lying on or near a manifold is not new. It was used to approximate a manifold (of known dimension) in \cite{aamari2019nonasymptotic} and \cite{cazals2005estimating}. \cite{https://doi.org/10.1111/rssb.12508} approximate manifolds with spherelets. \cite{tyagi2013tangent} consider the application of PCA to neighborhoods of a manifold modeled via quadratic embedding to analyze estimation of tangent space, yet does not consider ID estimation. However, to the best of our knowledge, this is the first time quadratic embeddings have been used in combination with PCA to estimate the ID of a manifold.

Our paper is outlined as follows. In Section \ref{sec:background}, we describe the version of local PCA found in \cite{lim2021tangent}. In Section \ref{sec: theory}, we compute the first-order behavior of eigenvalues expected for the covariance matrix of a quadratic embedding, which is then used to derive our test. The formal calculations are saved for Appendix \ref{sec:appendix}. Experiments on data sampled from manifolds, both synthetic and simulated, are described in Section \ref{sec:experiments}. Discussion is in Section \ref{sec:conclusion}.

\section{Background and Problem Setup}\label{sec:background}

Let $\{x_1,...,x_k\}$ be a sample of points in $\R^D$. The covariance matrix $\Sigma[x_1,...,x_k]$ of this sample is constructed as follows. Let
$\bar{x}=\frac{1}{k}\sum_{i=1}^k x_i$
and define
\begin{equation*}
    \hat{\Sigma}[x_1,...,x_k]=\frac{1}{k-1}\sum_{i=1}^k (x_i-\bar{x})(x_i-\bar{x})^T,
\end{equation*}
where $\bar{x}$ and each $x_i$ are interpreted as column vectors. A continuous version $\Sigma[\mu]$ may be computed for probability measures $\mu$ on $\R^D$ by replacing the summation with integration, or expectation. Specifically,
\begin{equation*}
    \Sigma[\mu]=\mathbb{E}(x-\bar{x})(x-\bar{x})^T=\int_{\R^{D}}(x-\bar{x})(x-\bar{x})^T d\mu(x),
\end{equation*}
where
\begin{equation*}
    \bar{x}=\int_{\R^{D}}x d\mu(x).
\end{equation*}

In the case where $\mu$ is an arbitrary finite, nonnegative measure, we replace $d\mu(x)$ with $\frac{d\mu(x)}{\int_{\R^D}d\mu(x)}$ in both the above expressions.

Observe that $\hat{\Sigma}[x_1,...,x_k]$ and $\Sigma[\mu]$ are always symmetric, positive semidefinite matrices. Let $\vec{\lambda}\hat{\Sigma}[x_1,...,x_k]$ denote a vector consisting of the eigenvalues of $\hat{\Sigma}[x_1,...,x_k]$ in decreasing order and similarly for $\vec{\lambda}\Sigma[\mu]$.

The idea behind local PCA for ID is that if $x_1,...,x_k$ are sampled from a small neighborhood where the underlying manifold is well-approximated by its $d$-dimensional tangent space, then one expects the first $d$ elements of $\vec{\lambda}\hat{\Sigma}[x_1,...,x_k]$ to be much larger than the last $D-d$ elements. There are many ways to translate this observation into practice; see citations in the Introduction for reference. Here, we focus on a formulation of the test described by~\cite{lim2021tangent} which has the following benefits. First, it requires no human judgment or arbitrary threshold cutoffs. Second, it presents the possibility of a simple modification to adjust for curvature of the underlying manifold. Lastly, it is supported by evidence from our experiments as well as proof of statistical convergence \cite{lim2021tangent}.

\begin{lemma}[Lemma 6.1 in \cite{lim2021tangent}, Lemma 13 in \cite{arias2017spectral}]\label{lemma:distribution for ball}
    Let $W$ be a $d$-dimensional subspace of $\R^D$ ($D\ge d$) and let $\nu$ denote the $d$-dimensional Lebesgue measure on $W$ intersected with the unit ball of $\R^D$. Then
    \begin{equation}\label{eq:lambdadD}
        \vec{\lambda}\Sigma[\nu]:=\vec{\lambda}(d,D):=\frac{1}{d+2}(\underset{d\text{ times}}{\underbrace{1,...,1}},\underset{D-d\text{ times}}{\underbrace{0,...,0}}).
    \end{equation}
\end{lemma}

An elementary argument shows that for $r>0$ and $v\in\R^D$,
\begin{equation*}
    \frac{1}{r^2}\vec{\lambda}\hat{\Sigma}[rx_1-v,...,rx_k-v]=\vec{\lambda}\hat{\Sigma}[x_1,...,x_k].
\end{equation*}
Thus, given points sampled from a $d$-dimensional ball of radius $r$ centered away from the origin in $\R^D$, we expect $1/r^2$ times the associated eigenvalues to be close to $\vec{\lambda}(d,D)$.

Let $X\subset\R^D$ be a collection of points, presumably on or near a $d$-dimensional manifold embedded in $\R^D$. Let $x\in X$ and $\{x_1,...,x_k\}$ be the neighbors of $x$ in $X$ lying within distance $r$ of $x$. Then, the test described in \cite{lim2021tangent} determines an estimated ID $\hat{d}$ at $x$ by
\begin{equation}\label{eq:PCA_method}
    \hat{d}=\argmin_{1\le d\le D}\left\|\frac{1}{r^2}\vec{\lambda}\hat{\Sigma}[x_1,...,x_k]-\vec{\lambda}(d,D)\right\|_2,
\end{equation}
where $\vec{\lambda}(d,D)$ is as in \eqref{eq:lambdadD}.

\section{Main Results and Proposed Method: CA-PCA}
\label{sec: theory}

In this section, we first state our main result on the first-order behavior of the eigenvalues of the covariance matrix for the uniform distribution on a Riemannian manifold of dimension $d$. Next, we show how this can be used to derive our main algorithm, CA-PCA, before discussing some of the proof, most of which will be contained in Appendix \ref{sec:appendix}.

\subsection{First-Order Behavior of Eigenvalues}

Given a $C^2$, $d$-dimensional manifold $\mathcal{M}\subset\R^D$ and a point $p\in \mathcal{M}$, there exists an orthonormal set of coordinates $(x_1,...,x_D)$ for $\R^D$ such that $\mathcal{M}$ is locally the graph of a $C^2$ function $F:\R^d\mapsto \R^{D-d}$. Without loss of generality, we take $p$ to be the origin in $\R^D$.

Since $F$ is well-approximated by its Taylor series of order 2, we consider the quadratic embedding
\begin{equation}\label{eq:parametrization}
    Q:(x_1,...,x_d)\mapsto (x_1,...,x_d,Q_1(x_1,...,x_d),...,Q_{D-d}(x_1,...,x_d)),
\end{equation}
where $Q_j:\R^d\to\R$ is a quadratic form of the form $Q_j(x)=x^TM_jx$ for a symmetric $d\times d$ matrix $M_j$ ($1\le j\le D-d$). We denote the eigenvalues of $M_j$ as $\lambda_{1,j},...,\lambda_{d,j}$. Denote the image of $Q$ as $\mathcal{M}_Q$ and the $d$-dimensional Riemannian volume form on $\mathcal{M}_Q$ by $d\mu_Q$. Let $\Sigma$ denote the covariance matrix of $d\mu_Q$ restricted to the unit ball.

Write
\begin{equation}
    \Sigma=\begin{bmatrix}
        \Sigma_{1} & \Sigma_{12}\\
        \Sigma_{12}^T &\Sigma_2,
    \end{bmatrix}
\end{equation}
where $\Sigma_1$ is a $d\times d$ matrix, $\Sigma_{12}$ is a $d \times (D-d)$ matrix, and $\Sigma_{2}$ is a $(D-d) \times (D-d)$ matrix. We will show in Subsection \ref{subsec:proof} that $\Sigma_{12}=0$, motivating the following proposition.

To simplify the following expressions, we define the shorthands 
\begin{equation*}
    A_j=\sum_{i=1}^d \lambda_{i,j}^2, \hspace{.1 in} A=\sum_{j=1}^{D-d}A_j=\sum_{i,j}\lambda_{i,j}^2
\end{equation*}
and
\begin{equation*}
    B_j=\left(\sum_{i=1}^d \lambda_{i,j}\right)^2, \hspace{.1 in} B=\sum_{j=1}^{D-d}B_j=\sum_{j}\left(\sum_i\lambda_{i,j}\right)^2.
\end{equation*}

Observe that each of $A,A_j, B,B_j$ is $O(\Lambda^2)$, where $\Lambda=\max_j\|M_j\|$.

\begin{proposition}\label{prop:main prop}
    Let $\Sigma_1, \Sigma_2$ be as above. Denote the eigenvalues of $\Sigma_1$ by $\lambda_1,...,\lambda_d$.
    Then,
    \begin{equation}\label{eq:upper trace formula}
       \tr(\Sigma_1)=\frac{d}{d+2}-\frac{2d}{(d+2)^2(d+4)}A-\frac{1}{(d+2)^2}B+O(\Lambda^4)
    \end{equation}
    and
    \begin{equation}\label{eq:sum lowest D-d}
        \tr(\Sigma_2)=\frac{2}{(d+2)^2}A+O(\Lambda^4).
    \end{equation}
    Furthermore, for $1\le i\le d$, 
    \begin{equation}\label{eq:upper bound for lambda_i}
        \lambda_i\ge \frac{1}{d+2}-\frac{20d+42}{(d+2)^2(d+4)}A-\frac{11d+20}{2(d+2)^2(d+4)}B+O(\Lambda^4)
    \end{equation}
and
\begin{equation}\label{eq:lower bound for lambda_i}
        \lambda_i\le \frac{1}{d+2}+\frac{5d+8}{(d+2)^2(d+4)}A+\frac{1}{2(d+2)^2}B + O(\Lambda^4).
    \end{equation}
\end{proposition}


The bounds in \eqref{eq:upper bound for lambda_i} and \eqref{eq:lower bound for lambda_i} are not sharp. Our method of proof ignores much potential cancellation; however, it still reveals that $|\frac{1}{d+2}-\lambda_i|\le O(d^{-2}\sum_{j=d+1}^{D}\lambda_j^2)$.

\subsection{Proposed Method: CA-PCA}\label{subsec:actual alg}

Since $B$ may not generally be calculated as a function of $A$, we have not yet achieved an explicit relation between $\tr(\Sigma_1)$ and $\tr(\Sigma_2)$. Absent any further assumptions, we are unable to find any such relation and complete a derivation of our test. We take our ``best guess'' of how they may relate to be the expectation of what occurs when the eigenvalues are of random sign.

Specifically, let $\epsilon_{1,j},...,\epsilon_{d,j}$ be i.i.d. random variables taking on value 1 and -1, each with probability 1/2. Let $\lambda_{i,j}=\epsilon_{i,j}\alpha_{i,j}$, where $\alpha_{i,1},...,\alpha_{i,d}$ are fixed. Then,
\begin{equation*}
    \mathbb{E}B_j=\mathbb{E}(\sum_{i=1}^d \epsilon_{i,j}\alpha_{i,j})^2=\mathbb{E}\sum_{1\le i_1,i_2\le d}\epsilon_{i_1,j}\epsilon_{i_2,j}\alpha_{i_1,j}\alpha_{i_2,j}=\sum_{i=1}^d\alpha_{i,j}^2=\sum_{i=1}^d\lambda_{i,j}^2=A_j.
\end{equation*}

Thus, for the purpose of deriving our test in Subsection \ref{subsec:actual alg} we will assume $A_j=B_j$ for all $1\le j\le D-d$. When there are plenty of quadratic forms $Q$ such that all $\lambda_{i,j}>0$ and others where $A_j>B_j=0$ for all $j$, the above assumption of randomly signed eigenvalues is a proxy for an ``average'' choice of $Q$. 

Substituting $A=B$ into \eqref{eq:upper trace formula},
\begin{equation}\label{eq:upper trace A=B}
    \tr(\Sigma_1)=\frac{d}{d+2}-\frac{3d+4}{(d+2)^2(d+4)}A
\end{equation}
with average eigenvalue
\begin{equation}\label{eq:avg upper eigenvalue}
    \frac{1}{d+2}-\frac{3d+4}{d(d+2)^2(d+4)}A.
\end{equation}



Proposition \ref{prop:main prop} provides a formula for $\sum_{i=1}^d\lambda_i=\tr(\Sigma_1)$. Ideally, we would like to know the individual values of $\lambda_i$ to compare to the eigenvalues coming from the sampled points. However, in practice, we are only given the sampled eigenvalues so we will assume that each $\lambda_i$ is equal. This is a somewhat reasonable assumption given the bounds in \eqref{eq:upper bound for lambda_i} and \eqref{eq:lower bound for lambda_i}. Furthermore, the sampled eigenvalues may differ from each other due to statistical noise, so it may be faulty to assume the difference is due to curvature effects in $\mathcal{M}$. In truth, $\lambda_1$ through $\lambda_d$ may still differ a little, but this approach allows us to most naturally extend \eqref{eq:PCA_method} to our proposed method.

Thus, we assume
\begin{equation*}
    \lambda_i=\frac{1}{d+2}-c_h(d)A, 1\le i\le d,
\end{equation*}
where $c_h(d)=\frac{3d+4}{d(d+2)^2(d+4)}$ by \eqref{eq:avg upper eigenvalue}. We assume $\Lambda$ is small so we may ignore the $O(\Lambda^4)$ terms.

Letting $c_l(d)=\frac{2}{(d+2)^2}$, rewrite \eqref{eq:sum lowest D-d} as
\begin{equation*}
    \sum_{j=d+1}^D\lambda_j=c_l(d)\sum_{j=1}^{D-d}A_{j}=c_l(d)A.
\end{equation*}

By substitution,
\begin{equation*}
    \frac{1}{d+2}=\lambda_i+\frac{c_h(d)}{c_l(d)}\sum_{j=d+1}^D\lambda_j=\lambda_i+c(d)\sum_{j=d+1}^D\lambda_j.
\end{equation*}
by setting $c(d)=c_h(d)/c_l(d)=(3d+4)/(2d(d+4))$.

Given $x\in X$, let $x_1,...,x_{k+1}$ denote the $k+1$ nearest neighbors of $x$, in increasing order of distance from $x$. Let $r = (\|x - x_k\|_2 + \|x - x_{k+1}\|_2)/2$ and determine sample eigenvalues $\hat{ \lambda}_1\ge...\ge \hat{\lambda}_D$ from the matrix $\frac{1}{r^2} \hat \Sigma [x_1,...,x_k]$. Statistically, we expect that as the number of samples increases, $(\hat\lambda_1,...,\hat\lambda_d)$ will converge to the eigenvalues of $\Sigma_1$ and $(\hat\lambda_{d+1},...,\hat\lambda_D)$ will converge to the eigenvalues of $\Sigma_2$.

When the neighborhood of points is taken at a reasonable scale relative to the curvature of the manifold, $\Lambda$ is small; hence, by Proposition \ref{prop:main prop} we expect the eigenvalues of $\Sigma_1$ to be greater than the eigenvalues of $\Sigma_2$. Given reasonable curvature and sample size, we may treat $(\hat\lambda_1,\ldots,\hat\lambda_d)$ as ``coming from'' $\Sigma_1$, $(\hat\lambda_{d+1},\ldots,\hat\lambda_D)$ as coming from $\Sigma_2$, and expect
\begin{equation}\label{eq:define modified lambda}
\lambda^{(d)}_i=\hat\lambda_i+c(d)\sum_{j=d+1}^D\hat{\lambda}_j\approx \frac{1}{d+2}, 1\le i\le d.
\end{equation}

Here, we make our first adjustment to the above theory. We note that the $(\frac{1}{d+2}\sum_{i=1}^d\lambda_{i,j}+O(\Lambda^2))^2$ term in \eqref{eq:for2xfactor} always leads to a strictly \emph{negative} coefficient on the 4th degree terms. Thus, we expect 
\begin{equation*}
    \sum_{j=d+1}^D\hat{\lambda}_j<\frac{2}{(d+2)^2}A.
\end{equation*}

To make up for this, we need to add more to the upper eigenvalues to obtain $\frac{1}{d+2}$, though this additional amount may be difficult to compute. For simplicity, we simply add double the amount as before and adjusting \eqref{eq:define modified lambda}, we now expect
\begin{equation}\label{eq:into_factor2}
    \lambda^{(d)}_i:=\hat\lambda_i+2c(d)\sum_{j=d+1}^D\hat{\lambda}_j\approx \frac{1}{d+2}, 1\le i\le d.
\end{equation}

Set $\lambda^{(d)}_j=0$ for $d+1\le j\le D$. Motivated by \eqref{eq:PCA_method}, one might simply choose $1\le d\le D$ to minimize the quantity
\begin{equation}\label{eq:tempting test error}
    \|\vec{\lambda}(d,D)-(\lambda^{(d)}_1,...,\lambda^{(d)}_D)\|_2=\sqrt{\sum_{i=1}^d\Big(\frac{1}{d+2}-\lambda^{(d)}_i\Big)^2},
\end{equation}
where $\lambda^{(d)}_i$ is as in \eqref{eq:into_factor2}.
However, a small amount of statistical noise can easily lead us to pick the wrong value of $d$.

For example, suppose our data are sampled from $\R^2$ and computing principal component analysis on them returns the eigenvalues $(0.21,0.15)$. Taking $d=2$, the value of \eqref{eq:tempting test error} is $\|(0.25,0.25)-(0.21,0.15)\|_2=.1077$. Taking $d=1$, then we add $14/10*0.15$ back to $0.21$ to get $(\lambda^{(1)}_1,\lambda^{(1)}_2)=(0.42,0)$, then compute \eqref{eq:tempting test error} to be $\|(1/3,0)-(.42,0)\|_2=0.08667$. We pick the dimension corresponding to the smaller error; thus, $d=1$.

However, the second eigenvalue 0.15 in this example corresponds to a very high curvature of our manifold. For the quadratic embedding $x\mapsto (x,cx^2)$, our model predicts $c=\sqrt{27/40}$, in which case we are working at larger scales relative to our manifold, violating small $\Lambda$ assumptions.

Thus, we choose
\begin{equation}\label{eq:error to be minimized}
    \hat{d}=\argmin_{1\le d\le D}\|\vec{\lambda}(d,D)-(\lambda^{(d)}_1,\ldots,\lambda^{(d)}_D)\|_2+2\sum_{j=d+1}^D \hat{\lambda}_j.
\end{equation}

One may consider this modification as analogous to LASSO, in which one attempts to minimize an $\ell^2$ error in conjunction with the size of the coefficients. The idea is that each $\hat{\lambda}_j$ ($d+1\le j\le D$) represents, in some sense, the total curvature of the manifold in the $j$-th coordinate, perhaps distributed according to some double-exponential priors. Introduction of the regularizer guards against the model assuming too large of a curvature (which would violate our small curvature assumption). The factor of 2 is introduced in light of the factor of 2 in \eqref{eq:into_factor2}. 




\begin{algorithm}
\caption{CA-PCA}\label{alg:CAPCA}
\begin{algorithmic}
\STATE \textbf{Input:} $X, N, k$ 
\FOR {$n = 1,...,N$}
    \STATE Sample $x \in X$ uniformly at random
    \STATE Choose $\{x_1,...,x_{k+1}\}$ $(k+1)$-nearest neighbors to $x$
    \STATE Set $r = (\|x - x_k\|_2 + \|x - x_{k+1}\|_2)/2$
    \STATE Form $\hat \Sigma [x_1,...,x_k]$
    \STATE Calculate eigenvalues
    \[
        (\hat \lambda_1,...,\hat \lambda_D) = 
            \frac{1}{r^2}\vec{\lambda} \hat \Sigma [x_1,...,x_k]
    \]
    \FOR {$d=1,...,D$}
        \STATE Compute $(\lambda^{(d)}_1,...,\lambda^{(d)}_D)$ by \eqref{eq:into_factor2}:
        \begin{equation*}        \lambda^{(d)}_i:=\hat\lambda_i+(3d+4)/(d(d+4))\sum_{j=d+1}^D\hat{\lambda}_j \text{ for } 1\le i\le d
        \end{equation*}
        and
        \begin{equation*}
            \lambda_i^{(d)}:=0 \text{ for } d+1\le i\le D.
        \end{equation*}
    \ENDFOR
    \STATE Select
        \begin{align*}
            \hat d^{(n)} &= \argmin_{1 \leq d \le D}
           \|\vec{\lambda}(d,D)-(\lambda^{(d)}_1,\ldots,\lambda^{(d)}_D)\|_2+2\sum_{j=d+1}^D \hat{\lambda}_j.
        \end{align*}
\ENDFOR
\STATE Return {\sc Mean} $(\hat d^{(1)},...,\hat d^{(N)})$
\end{algorithmic}
\end{algorithm}

In Algorithm \ref{alg:CAPCA}, we could sample every point of $X$ once before averaging the results, but this may be computationally expensive. We could also sample a single point, though this may raise the possibility of being misled by an outlier. As a compromise, we choose to sample a random number of points, depending on how the user decides to balance their available computational power with their desired level of confidence.

Another parameter to be determined by the user is the number of nearest neighbors $k$. Here, the challenge is to balance the possibility of the $k$ nearest neighbors starting to ``wrap around'' the manifold for large $k$ with the small sample size one has for small $k$. While CA-PCA is designed to adjust for curvature, that adjustment has its limits. Large values of $k$ would potentially make for longer computation of principal component analysis; however, this effect should be negligible for the values of $k$ chosen in our experiments.

We make no claims about a ``best'' way to choose $k$, except that one would expect it to depend on prior assumptions about the size of the curvature (risk of wrap-around) and the value of $N$ (for perhaps offering some counterbalance to small sample size). However, in Section \ref{sec:experiments} we will run experiments with varying values of $k$ so one may compare methods with different values and perhaps begin to deduce patterns about best choice.

\subsection{Proof of Proposition \ref{prop:main prop}}\label{subsec:proof}

Our goal is to compute the eigenvalues of the covariance matrix for the uniform (with respect to $\mu_Q$) distribution on $\mathcal{M}_Q\cap B_r(0)$ for some radius $r>0$. By rescaling, take $r=1$. We denote this matrix $\Sigma$. For this purpose, we define $S:\R^d\to \R$, the density function with respect to $d\mu_Q$, as $d\mu_Q(x_1,...,x_d)=S(x_1,...,x_d)dx$, where $dx$ is the Lebesgue measure on $\R^d$.

Define quadratic forms $\tilde{Q}_j(x)=x^TM_j^2x$, and note that the eigenvalues of the matrix $M_j^2$ are $\lambda_{j,1}^2,...,\lambda_{j,d}^2$.


\begin{lemma}\label{lemma:form for S(x)}
Under the above assumptions and notation, the density function of $d\mu_Q$ is
    \begin{equation}\label{eq:form for S(x)}
    S(x)= 1+2\sum_{j=1}^{D-d}\tilde{Q}_j(x)+O(\Lambda^4).
    \end{equation}
\end{lemma}

\begin{proof}
Given \eqref{eq:parametrization},
\begin{equation*}
    S(x)=\sqrt{|\det(G(x))|},
\end{equation*}
where
\begin{equation*}
    G(x)=J^T(x)J(x)
\end{equation*}
and
\begin{equation*}
    J(x)=\begin{bmatrix}
        I_d\\
        \nabla Q_1(x)^T\\
        \vdots\\
        \nabla Q_{D-d}(x)^T
    \end{bmatrix}=\begin{bmatrix}
        I_d\\
        2(M_1x)^T\\
        \vdots\\
        2(M_{D-d}x)^T
    \end{bmatrix}.
\end{equation*}

Thus, taking advantage of the symmetry of the $M_j$,
\begin{equation}\label{eq:matrix entries}
    G(x)=I_d+4\sum_{j=1}^{D-d}M_jxx^TM_j.
\end{equation}

As a result,
\begin{align*}
    \tr (G(x))&= d+4\sum_{j=1}^{D-d}\tr(M_jxx^TM_j)\\
    &=d+4\sum_{j=1}^{D-d} x^TM_j^2x\\
    &=d+4\sum_{j=1}^{D-d}\tilde{Q}_j(x).
\end{align*}






By \eqref{eq:matrix entries}, $G(x)-I_d$ has entries of order $O(\Lambda^2)$. Thus, the eigenvalues of $G(x)$ each differ from 1 by $O(\Lambda^2)$. That is, for fixed $x_0$, and denoting the eigenvalues of $G(x)$ by $\mu_1,...,\mu_d$,
\begin{equation*}
    \mu_i=1+c_i,
\end{equation*}
where 
$c_i=O(\Lambda^2)$
and
\begin{equation*}
    \sum_{i=1}^d c_i=4\sum_{j=1}^{D-d}\tilde{Q}_j(x_0).
\end{equation*}

Thus,
\begin{equation*}
    \det(G(x))=\prod_{i=1}^d(1+c_i)=1+4\sum_{j=1}^{D-d}\tilde{Q}_j(x)+O(\Lambda^4).
\end{equation*}

By the Taylor expansion $\sqrt{1+t}=1+\frac{1}{2}t+O(t^2)$, we have the desired conclusion.
\end{proof}

Recall, our goal is to compute the eigenvalues of the covariance matrix $\Sigma$.
Let
\begin{equation*}
    R=\{x\in\R^d:|x|^2+\sum_{j=1}^{D-d}(Q_j(x))^2\le 1\}
\end{equation*}
and
\begin{equation*}
    \tilde{S}(x)=\frac{S(x)}{\int_R S(x)dx}.
\end{equation*}

By definition,
\begin{equation}\label{eq:first_integral}
    (\Sigma_1)_{i,j}=\int_R (x_i-\bar{x}_i)(x_j-\bar{x}_j) \tilde{S}(x)dx,
\end{equation}
\begin{equation*}
    (\Sigma_{12})_{i,j}=\int_R (x_i-\bar{x}_i)(Q_j(x)-\bar{Q}_j)\tilde{S}(x)dx,
\end{equation*}
and
\begin{equation}\label{eq:second_integral}
    (\Sigma_2)_{i,j}=\int_R (Q_i(x)-\bar{Q}_i)(Q_j(x)-\bar{Q}_j)\tilde{S}(x)dx.
\end{equation}

In the above,
\begin{equation*}
    \bar{x}_i=\int_R x_i\tilde{S}(x)dx, \hspace{.25 in}
    \bar{Q}_i=\int_R Q_i(x)\tilde{S}(x)dx.
\end{equation*}


Since $x_i$ is an odd function and $R$ is a region symmetric about the origin, we immediately see that $\bar{x}_i=0$ for $1\le i\le d$. Furthermore, $Q_j(x)-\bar{Q}_j$ is an even function and $x_i-\bar{x}_i=x_i$ is odd, so $\Sigma_{12}=0$. As a result, $\Sigma$ is a block diagonal matrix with an upper block $\Sigma_1$ of size $d\times d$ and a lower block $\Sigma_2$ of size $(D-d)\times(D-d)$, that is

\begin{equation}
\Sigma=\left[\begin{array}{c|ccc}
\Sigma_1 & &0 &\\
\hline
& & &\\
0& & \Sigma_2 &\\
& & &
\end{array}
\right].
\end{equation}

In order to determine $\tr(\Sigma_1)$ and $\tr(\Sigma_2)$, we only need the diagonal entries of $\Sigma$, that is, the integrals in \eqref{eq:first_integral} and \eqref{eq:second_integral}. To determine the bounds on $\lambda_1,...,\lambda_d$, we compute \eqref{eq:first_integral} with $x_i$ replaced by an arbitrary direction which may be chosen to maximize or minimize the integral. We perform these computations in Appendix \ref{sec:appendix}. The derivation is long, but it is a straightforward application of integration in radial coordinates, integration of polynomial functions over the sphere, truncation of Taylor series, and Lemma \ref{lemma:form for S(x)}.

In \cite{tyagi2013tangent}, the authors consider separately the ``uncorrelated case,'' corresponding to taking $(\Sigma_2)_{i,j}=0$ for all $1\le i\neq j\le D-d$ above. One could use this assumption to calculate the eigenvalues of $\Sigma_2$ explicitly ($\frac{2}{(d+2)^2}A_j$, as will be shown in the proof of Proposition \ref{prop:main prop}); however, only the sum of these eigenvalues will be needed for comparison with the eigenvalues of $\Sigma_1$, so we will not need to make this assumption here.

\section{Experiments}\label{sec:experiments}

\subsection{Outline}

We ran CA-PCA (\eqref{eq:error to be minimized}) on a number of examples of point clouds, each contained in $\R^D$ for some $D$ and (presumably) sampled from manifolds. The results were compared with the approach from \cite{lim2021tangent} (PCA), seen in \eqref{eq:PCA_method}, and the maximum likelihood estimator from \cite{levina2004maximum} (LB).

To fully compare CA-PCA to all other existing, competitive ID estimators would be onerous. \cite{camastra2016intrinsic} provides five criteria for a successful estimator: computational feasability, robustness to multiscaling, robustness to high (intrinsic) dimension, a work envelope, and accuracy. Among eleven estimators, none is the clear favorite as different tests satisfy different criteria. For this reason, and to reflect our main insight, we emphasize a comparison between CA-PCA and PCA. However, we still compare CA-PCA with LB to demonstrate that our estimator is at least competitive with another popular method.


Given a point cloud $X\subset\R^D$, a random point of $X$ was sampled, denoted $x$. The $k$ nearest neighbors of $x$ were computed and used to run principal component analysis, from which $D$ eigenvalues were obtained. These eigenvalues were normalized by multiplying by $1/r^2$, where $r$ was chosen by taking the arithmetic mean of the distance of the $k$-th nearest neighbor of $x$ and the distance of the $(k+1)$-st nearest neighbor of $x$. The CA-PCA test was run to determine an estimated ID. This process was repeated by sampling $N$ points randomly with replacement from $X$ and all of the estimated dimensions were averaged to produce a final estimate. An average estimated ID was then computed over a range of values of $k$, although the same points $x$ were sampled for each $k$ to reduce time spent computing the nearest neighbors. Our results are depicted by graphing the averaged estimated ID as a function of $k$ and comparing to the actual ID $d$. $N=200$ unless otherwise stated.

The same process was used for the PCA and LB tests, except for the LB test no computation or rescaling of eigenvalues was needed, merely distances between points. Also, the same randomly sampled points $x$ were used for each test.

In the case of point clouds in $\R^D$ for large $D$ (see Isomap faces and airplane photos below), the $D-k$ tail eigenvalues equal to zero were removed, though this had no practical effect on the algorithm other than to remove the possibility that $\hat{d}>k$.

When possible, manifolds are chosen to be without boundary, as points sampled from boundaries may produce a consistent underestimation of the ID, introducing a bias which may interfere with simple comparison of PCA and CA-PCA. In particular, the eigenvalues obtained from applying principal component analysis to a half-ball are precisely equal to those of the full ball, except one is smaller. If that smaller eigenvalue is treated as noise, the estimated dimension will be one lower than desired. CA-PCA is designed to adapt to effects of curvature, but not this scenario. (However, we will see it does better at handling the boundary issue than PCA and LB, at least in a couple of examples.)

\subsection{Some Synthetic Data}

One possible parametrization for a Klein bottle embedded in $\R^4$ is given by
\begin{align*}
    x&=(a+b\cos(v))\cos(u)\\
    y&=(a+b\cos(v))\sin(u)\\
    z&=b\sin(v)\cos(u/2)\\
    w&=b\sin(v)\sin(u/2)
\end{align*}
for $u,v\in[0,2\pi)$. Choosing $a=10$ and $b=5$, 400 points were sampled via the uniform distribution for $u,v$ on $[0,2\pi)^2$. While this does not correspond exactly to the uniform distribution on the Klein bottle, it does correspond to a smooth density. A smooth density may be considered locally constant, which would be good enough for any method which takes as input nearby points from the manifold, such as the three we are considering.

Running the three tests on these 400 points, we see that both PCA and CA-PCA appear to converge to the correct dimension of 2, although the latter is much faster (see Figure \ref{fig:Klein}). The LB estimator fails to converge, though it is much closer to 2 than 3. On this note, it is also clear that CA-PCA provides an estimated dimension closer to 2 than 1 for smaller values of $k$ than PCA.

\begin{figure}[hbt!]
    \centering
    \begin{subfigure}[t]{0.3\columnwidth}
         \centering
         \includegraphics[width=\columnwidth]{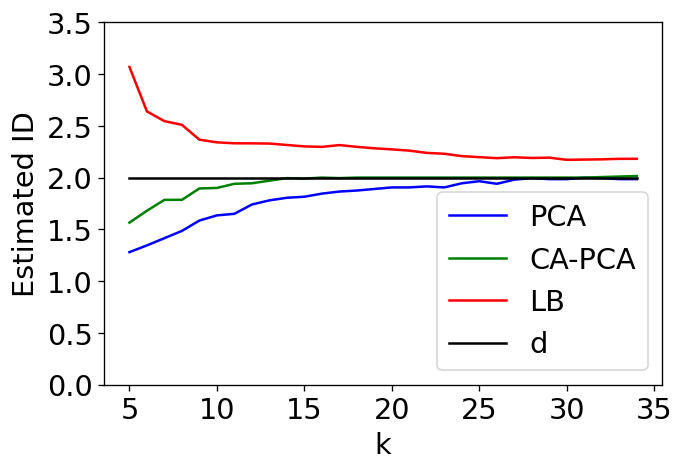}
         \caption{Estimated ID for Klein bottle in $\R^4$}
         \label{fig:Klein}
     \end{subfigure}
     \hfill
     \begin{subfigure}[t]{0.3\columnwidth}
         \centering
         \includegraphics[width=\columnwidth]{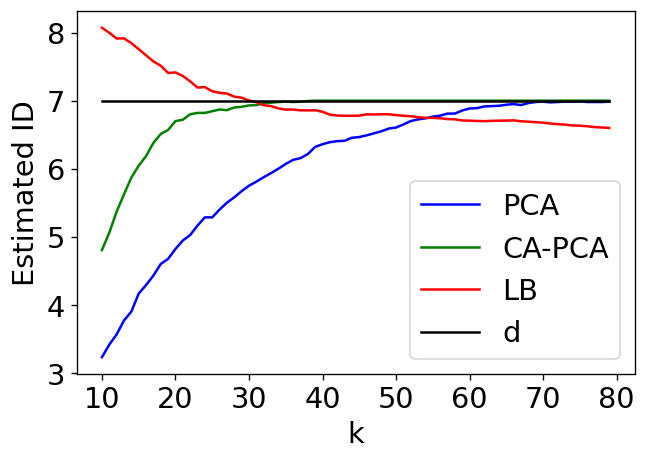}
         \caption{Estimated ID for $S^7\subset\R^8$}
         \label{fig:7-sphere}
     \end{subfigure}
     \hfill
     \begin{subfigure}[t]{0.3\columnwidth}
         \centering
         \includegraphics[width=\columnwidth]{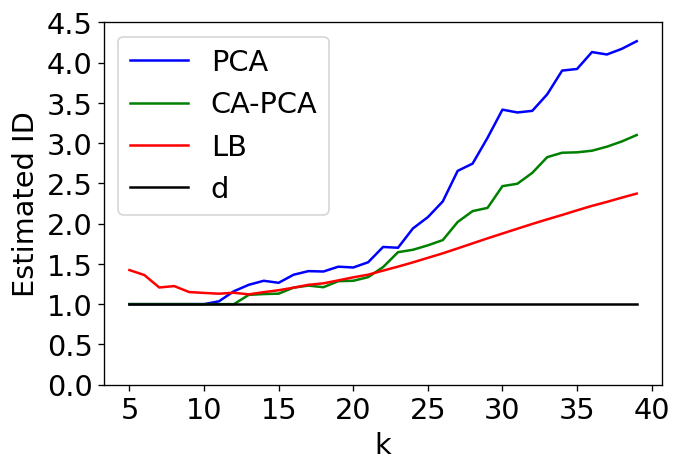}
         \caption{Estimated ID for curve in $\R^8$}
         \label{fig:curve in R^8}
     \end{subfigure}
        \caption{Results for Synthetic Data}
        \label{fig:synthetic}
\end{figure}

We randomly sampled 5,000 points from the unit sphere $S^7\subset\R^8$ with respect to the uniform measure. Here, CA-PCA converges to the true dimension of 7 much faster than PCA (see Figure \ref{fig:7-sphere}). While LB reaches the true dimension first, the estimate decreases (gets worse) as $k$ increases.

For $S^7$ and the Klein bottle, CA-PCA consistently produces a higher estimated ID than PCA. One may suspect that this behavior partially explains some of the faster convergence; however, the next example shows that CA-PCA can produce a lower estimate than PCA.

Consider the parametrization of a curve in $\R^8$ via the map
\begin{equation}\label{eq:fancy_curve}
\theta\mapsto(\cos\theta,\sin\theta,\cos2\theta,...,\cos4\theta,\sin4\theta), \theta\in[0,2\pi).
\end{equation}

A point cloud in $\R^8$ was determined by sampling 100 points from this curve, uniformly randomly over $\theta\in[0,2\pi)$. The results of the tests are shown in Figure \ref{fig:curve in R^8}. As the number of nearest neighbors increases and the collection of nearest neighbors wraps further around the curve, each of the tests provides an estimated ID increasingly higher than the true ID of 1. However, CA-PCA does so at a slower rate than PCA and provides a lower ID than PCA for each value of $k$. This is precisely what we expected CA-PCA to do: When the second-largest or third-largest eigenvalue increases, the method is more able to detect that it comes from the curvature of the manifold, rather than the dimension. Moving the mass of these lower eigenvalues onto the higher eigenvalue(s) is recognized as consistent with what we expect from a manifold. PCA, on the other hand, cannot make this recognition and is more likely to (falsely) determine that the manifold is of higher dimension.

This begs the question of how CA-PCA is able to (correctly) provide higher estimates of dimension for the examples of the Klein bottle and $S^7$. We suspect that the following occurs. At small values of $k$, the nearest neighbors are unlikely to fully vary in each direction of the tangent space of the manifold, leading to an eigenvalue or two which is lower than would otherwise be expected. PCA ``thinks'' that these lower eigenvalues are due to curvature, whereas CA-PCA is properly calibrated to determine they are not. CA-PCA ``knows'' that if these eigenvalues are from curvature they should relate to the other eigenvalues in a particular manner. The regularizer punishes the treatment of such eigenvalues as curvature if they are too large.

We next move to some more difficult examples, where both the dimension and codimension are higher than in most previous examples. In Figure \ref{fig:SO(5)}, we have the results of the tests applied to 20,000 points sampled randomly from $SO(5)$ with respect to the Haar measure and viewed as 5x5 matrices (thus as elements of $\R^{5\times5}\cong \R^{25}$). Here, CA-PCA is consistently a slightly better estimator than PCA for all values of $k\le 50$ before both agree with the actual ID.

\begin{figure}[hbt!]
    \centering
    \begin{subfigure}[t]{0.45\columnwidth}
         \centering
         \includegraphics[width=\columnwidth]{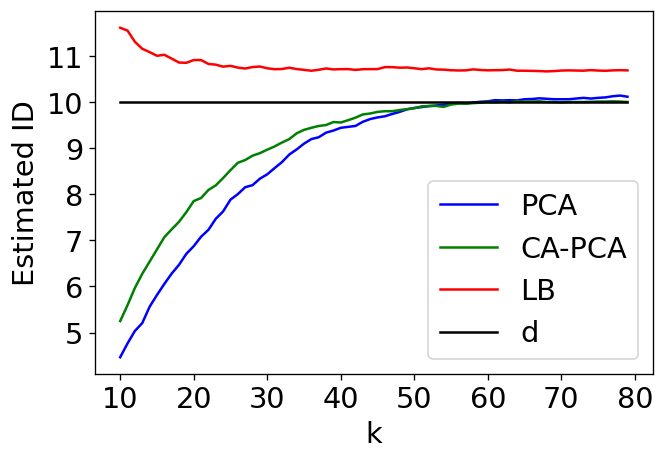}
         \caption{Estimated ID for $SO(5)$}
         \label{fig:SO(5)}
     \end{subfigure}
     \hfill
     \begin{subfigure}[t]{0.45\columnwidth}
         \centering
         \includegraphics[width=\columnwidth]{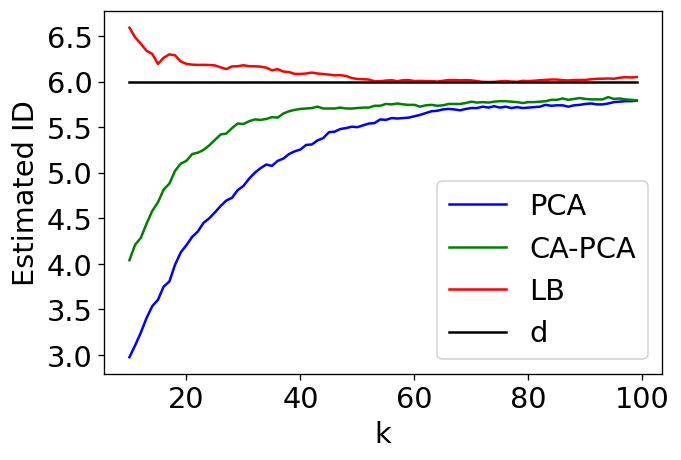}
         \caption{Estimated ID for a 6-dimensional Lie group}
         \label{fig:Lie group 6d}
     \end{subfigure}
     \caption{Synthetic Manifolds in Higher Dimensions}
        \label{fig:higher-d}
\end{figure}

The next example is a Lie group embedded in $\R^{15}$, constructed by taking the direct sum of the three-dimensional $SO(3)$ viewed as a subset of $\R^{3\times3}\cong \R^9$ and the 3-torus viewed as a subset of $\R^6$ through the parameterization
\begin{equation*}
    (\theta_1,\theta_2,\theta_3)\mapsto (\sin\theta_1,\cos\theta_1,\sin\theta_2,\cos\theta_2,\sin\theta_3,\cos\theta_3).
\end{equation*}

Results for 20,000 randomly sampled points are depicted in Figure \ref{fig:Lie group 6d}. CA-PCA fails to converge to the true dimension of 6, although it remains close even for small values of $k$ and outperforms PCA at every value of $k$ while competing with LB.

In Figure \ref{fig:higher-d}, we have plotted the average estimated dimension over a large number of points. Since there is plenty of ``wiggle room'' with the dimension and codimension so high, one may ask if the estimators consistently approximate the true dimension, or if they do so in the average only after providing alternating over- and underestimates. Figure \ref{fig:higher-d std dev} shows the standard deviation of the estimated dimensions for each test at fixed values of $k$. Since these are particularly small for both $SO(5)$ and $SO(3)\oplus \mathbb{T}^3$, the estimators, in particular PCA and CA-PCA, have consistently approximated the true ID in these cases.

While we do not include equivalents of Figure \ref{fig:higher-d std dev} for all our other experiments for the sake of brevity, we did construct them. In every case we examined, we observed a similar pattern where the standard deviation of each estimator was roughly proportional to the difference between the average estimated dimension and the actual dimension. It was also generally smaller for PCA and CA-PCA than for LB.

\begin{figure}[hbt!]
    \centering
    \begin{subfigure}[t]{0.45\columnwidth}
         \centering
         \includegraphics[width=\columnwidth]{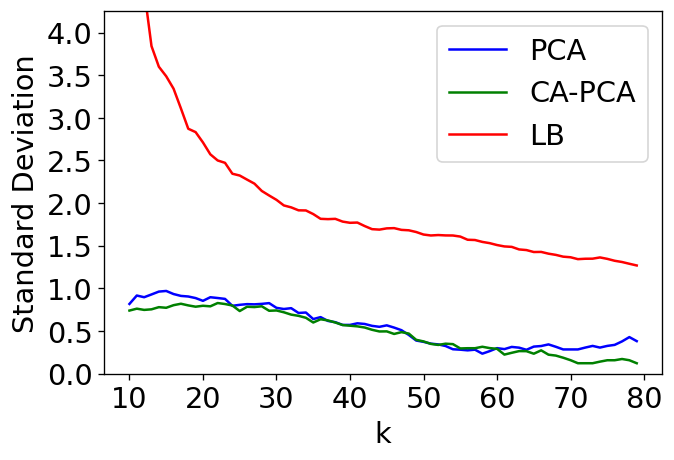}
         \caption{Standard deviation of estimated ID for $SO(5)$}
         \label{fig:SO(5) std dev}
     \end{subfigure}
     \hfill
     \begin{subfigure}[t]{0.45\columnwidth}
         \centering
         \includegraphics[width=\columnwidth]{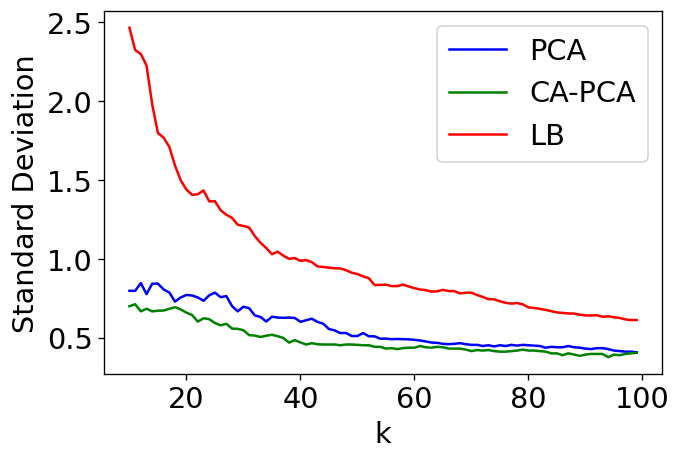}
         \caption{Standard deviation of estimated ID for a 6-dimensional Lie group}
         \label{fig:Lie group 6d std dev}
     \end{subfigure}
     \caption{Standard Deviation of Estimates for Synthetic Manifolds in Higher Dimensions}
        \label{fig:higher-d std dev}
\end{figure}

Figure \ref{fig:higher-d ou} shows for each of the three methods the proportions of estimates which were equal to the correct dimension, over it, or under it. For LB, which does not usually output an integer, we rounded the estimate to the nearest integer before checking if it hit the correct value. The results are consistent with what one might expect given the mean and standard deviation results. PCA and CA-PCA correctly estimate the dimension more as $k$ increases, underestimating it less frequently. At all values of $k$, they are very unlikely to overestimate the dimension. LB, on the other hand, tends to both over- and under-estimate the dimension for all values of $k$ and never gets the correct dimension at more than a 60\% rate.

\begin{figure}[hbt!]
    \centering
    \begin{subfigure}[t]{0.45\columnwidth}
         \centering
         \includegraphics[width=\columnwidth]{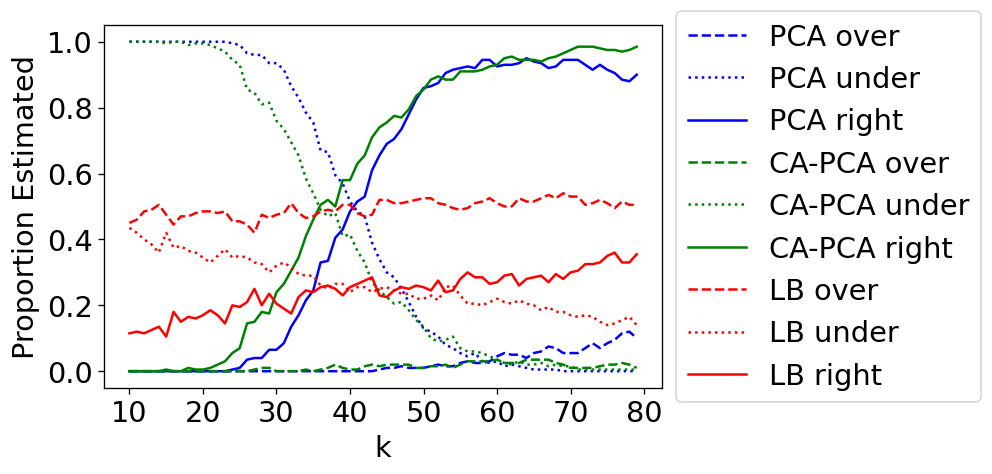}
         \caption{Proportion of Estimated Dimensions Correct/Over/Under for $SO(5)$}
         \label{fig:SO(5) ou}
     \end{subfigure}
     \hfill
     \begin{subfigure}[t]{0.45\columnwidth}
         \centering
         \includegraphics[width=\columnwidth]{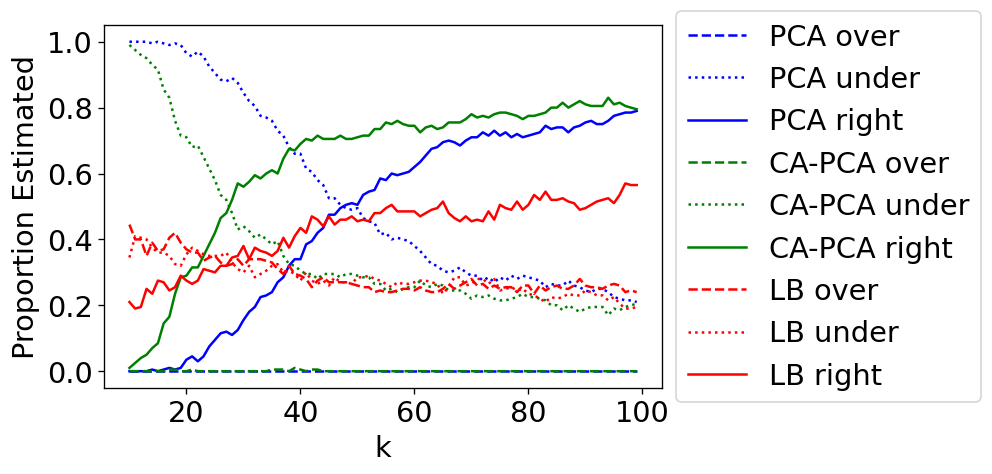}
         \caption{Proportion of Estimated Dimensions Correct/Over/Under for a 6-dimensional Lie group}
         \label{fig:Lie group 6d ou}
     \end{subfigure}
     \caption{Standard Deviation of Estimates for Synthetic Manifolds in Higher Dimensions}
        \label{fig:higher-d ou}
\end{figure}

\subsection{Some Simulated Data}\label{subsec:simulated}

An .stl file (3D image) of an airplane was rotated at a random angle around the $z$-axis before being projected onto the $x$-$z$ plane to determine a two-dimensional ``photograph'' of the airplane. Repeating this process, we obtained 200 images of size 432 by 288 pixels, viewed as vectors in $\R^{432\times288}\cong \R^{124,416}$. (A similar approach was taken in  \cite{Yariv}.) The expected dimension of this point cloud is 1, given that the images were generated via a 1-dimensional group of symmetries. Results are depicted in Figure \ref{fig:1D airplane}.

\begin{figure}[hbt!]
    \centering
    \begin{subfigure}[b]{0.3\columnwidth}
         \centering
         \includegraphics[width=\columnwidth]{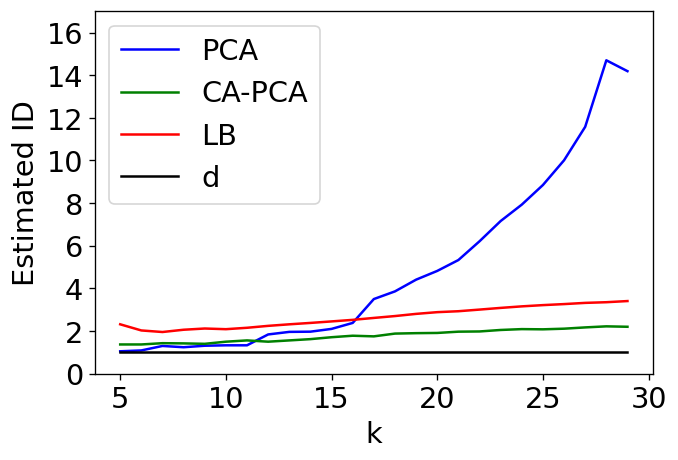}
         \caption{Estimated ID for 1-D airplane photos, $N=100$}
         \label{fig:1D airplane}
     \end{subfigure}
     \hfill
     \begin{subfigure}[b]{0.3\columnwidth}
         \centering
         \includegraphics[width=\columnwidth]{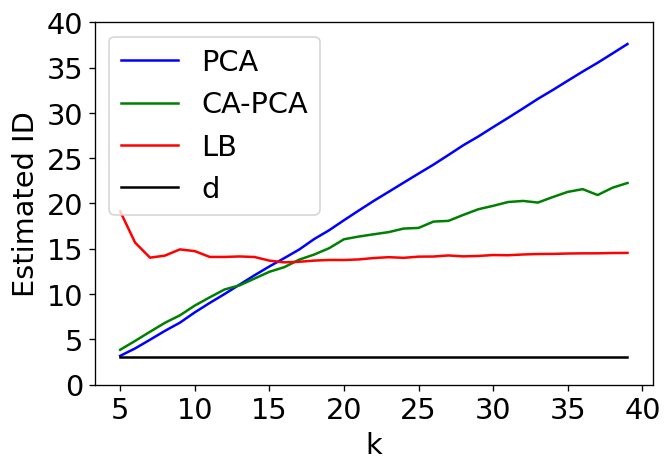}
         \caption{Estimated ID for 3-D airplane photos, $N=100$}
         \label{fig:3D airplane}
     \end{subfigure}
     \hfill
     \begin{subfigure}[b]{0.3\columnwidth}
         \centering
         \includegraphics[width=\columnwidth]{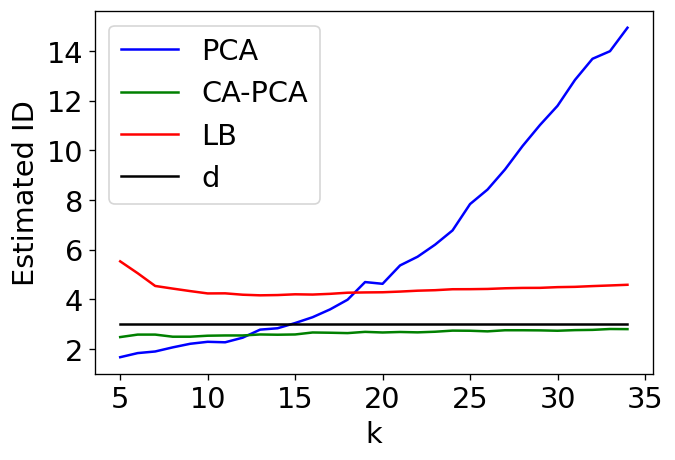}
         \caption{Estimated ID for Isomap faces}
         \label{fig:isomap faces}
     \end{subfigure}
        \caption{Results for Simulated Data}
        \label{fig:simulated}
\end{figure}

As one can see, both PCA and CA-PCA provide close to the ``correct'' answer for small values of $k$; however, for $k>13$, the estimated dimension for PCA blows up while for CA-PCA it remains between 1 and approximately 2. Considering only PCA, one may suspect that the point cloud does not lie on or near a manifold, or consider an estimated dimension so far from the true value that the estimator does not provide any benefits. In contrast, with CA-PCA the results are certainly consistent with data coming from a manifold. While the estimated dimension may be 2, a small amount of error assumed by the user might give them the range of 1 to 3, which may be sufficient.

We obtained 10,000 images of an airplane by applying a random element of $SO(3)$ (uniformly with respect to Haar measure) before projecting into two dimensions. The results of the ID estimators are shown in Figure \ref{fig:3D airplane}. Here, it is clear that CA-PCA is not always able to magically estimate the dimension in such cases; however, it does a better job of indicating to the user that there may be some underlying structure to the data than PCA does.

The Isomap face database consists of 698 greyscale images of size 64 by 64 pixels and has been used in \cite{levina2004maximum,lin2008riemannian,tenenbaum2000global} and many other works. The images are of an artificial face under varying illumination, vertical orientation, and horizontal orientation. Translating the images into vectors in $\R^{64\times64}\cong\R^{4096}$, we obtain the results in Figure \ref{fig:isomap faces}.

One may expect the true ID to be 3, due to the three parameters which vary in the construction of the photos. However, each of those three parameters has a clear upper and lower bound. For instance, there are photos with the head facing left, right, and center, but none with it facing backwards. By similar reasoning regarding the vertical orientation and brightness, we expect the Isomap faces to have similar structure as a three-dimensional unit cube.

The problem of determining the dimension of a manifold with boundary presents challenges distinct from those arising in dealing with manifolds without boundary. See \cite{berry2017density,carter2007biasing} for more on this problem.

For comparison, we construct a point cloud through random selection of 698 points from the unit cube in $\R^3$ and run the same tests, producing Figure \ref{fig:unit cube}. Here, PCA and LB predict a dimension a little under 3, while CA-PCA predicts almost exactly 3. 

\begin{figure}[hbt!]
    \centering
    \begin{subfigure}[b]{0.3\columnwidth}
         \centering
         \includegraphics[width=\columnwidth]{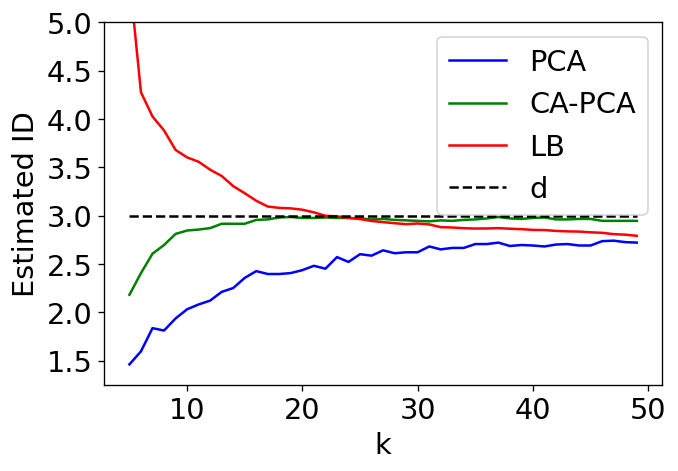}
         \caption{Estimated ID for unit cube in $\R^3$}
         \label{fig:unit cube}
     \end{subfigure}
     \hfill\begin{subfigure}[b]{0.3\columnwidth}
         \centering
         \includegraphics[width=\columnwidth]{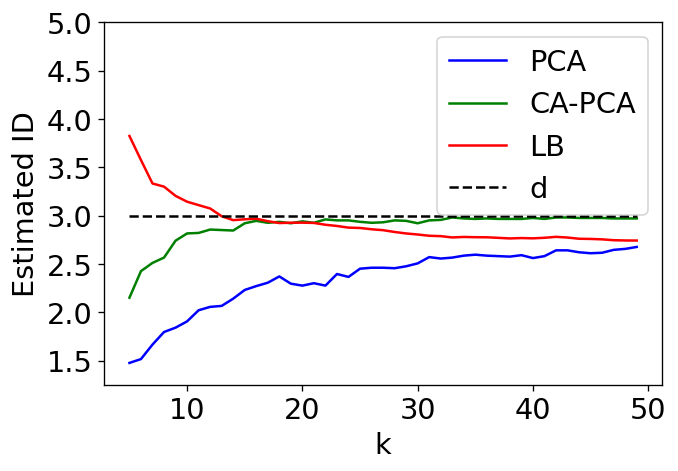}
         \caption{Estimated ID for a rectangular prism}
         \label{fig:prism}
     \end{subfigure}
     \hfill
     \begin{subfigure}[b]{0.3\columnwidth}
         \centering
         \includegraphics[width=\columnwidth]{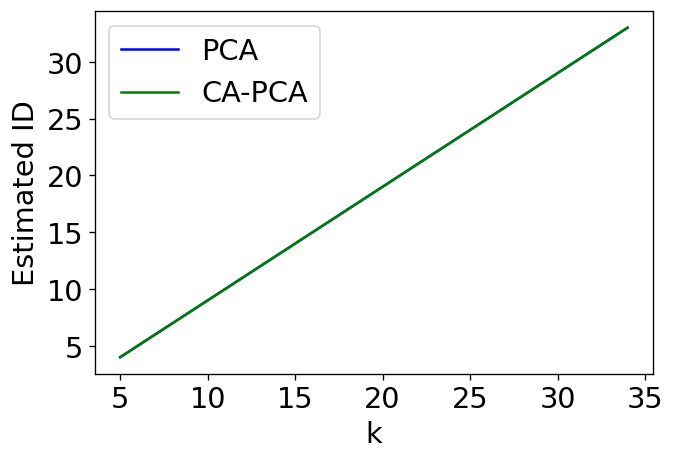}
         \caption{Estimated ID for random images in $\R^{64\times64}$}
         \label{fig:698 random}
     \end{subfigure}
     \caption{For Comparison with Isomap Faces}
        \label{fig:boundary}
\end{figure}

While the variations in the vertical and horizontal orientations of the face may be comparable in magnitude, it is unclear whether the variation in the brightness is the same. For this reason, we repeat the above with the unit cube replaced by $[0,1]\times[0,1]\times[0,5]$, with results in Figure \ref{fig:prism}. Again, we achieve similar results as in Figure \ref{fig:unit cube}

We note that CA-PCA does tend to get much closer to the `true' dimension of 3, if one chooses to assign this value to a manifold with boundary. It is possible this is no coincidence; that given a sampled point at the center of a face of the unit cube, CA-PCA is better able to determine that the third and lowest eigenvalue is \textit{not} due to curvature since this would require the first two eigenvalues to be smaller than observed. However, more investigation is needed to determine if CA-PCA truly analyzes manifolds with boundary differently.

Given the flexibility CA-PCA has to ``move tail eigenvalues over to higher ones,'' one may wonder if it is better able to estimate the dimension of the Isomap faces (and other sets of images) by consistent underestimation in the setting $k<<D$. To address this potential issue, we generated 698 random 64x64 grey-scale images (elements of the unit cube in $\R^{4096}$) and applied the three ID estimators. Results for PCA and CA-PCA are shown in Figure \ref{fig:698 random}. (The LB estimator output values over 300 so its results are not depicted for readability of the graph.) Here, CA-PCA provides the same estimates as PCA, suggesting its performance on the Isomap faces was truly due to proper estimation of the dimension rather than consistent underestimation.

\subsection{Analysis of Non-Manifolds}

Next, we see what happens when we apply the tests to another object which is not a manifold, in this case the union of two manifolds of different dimensions. We sample 200 points from the unit sphere in $\R^4$ and 200 points from the circle determined by the equations $(x_1-4)^2+x_2^2=16, x_3=0, x_4=0$.

Running tests on the union of the two datasets gives Figure \ref{fig:S1 U S3}. Considering the average, it appears the estimators tend to a value of 2, the average of the dimension of the 3-sphere and that of the circle. However, the standard deviation of the estimates shown in Figure \ref{fig:S1 U S3 std dev} demonstrates that there are many occurrences of 1's and 3's. This makes sense, given that base points sampled from the sphere (and a reasonable distance from the circle) would only have nearest neighbors also coming from the sphere, leading to a dimension estimate of 3. Likewise, points sampled from the circle (and reasonably far from the sphere) would lead to a dimension estimate of 1. For comparison, consider the standard deviations for the Klein bottle, shown in Figure \ref{fig:Klein std dev}.

\begin{figure}[hbt!]
    \centering
    \begin{subfigure}[t]{0.3\columnwidth}
         \centering
         \includegraphics[width=\columnwidth]{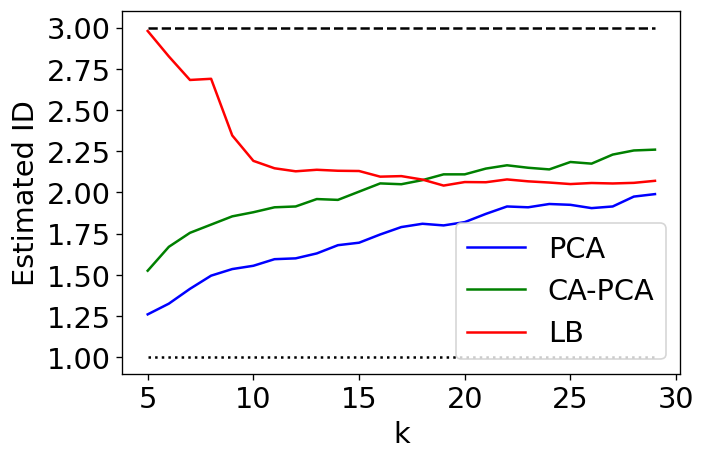}
         \caption{Estimated ID for union of circle and $S^3$ with dimensions of circle (1) and $S^3$ (3)}
         \label{fig:S1 U S3}
     \end{subfigure}
     \hfill
     \begin{subfigure}[t]{0.3\columnwidth}
         \centering
         \includegraphics[width=\columnwidth]{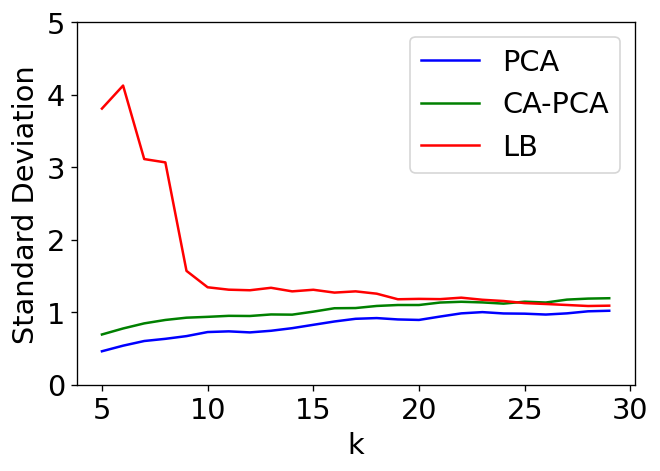}
         \caption{Standard deviation of estimated ID for the union of circle and $S^3$}
         \label{fig:S1 U S3 std dev}
     \end{subfigure}
     \hfill
    \begin{subfigure}[t]{0.3\columnwidth}
         \centering
         \includegraphics[width=\columnwidth]{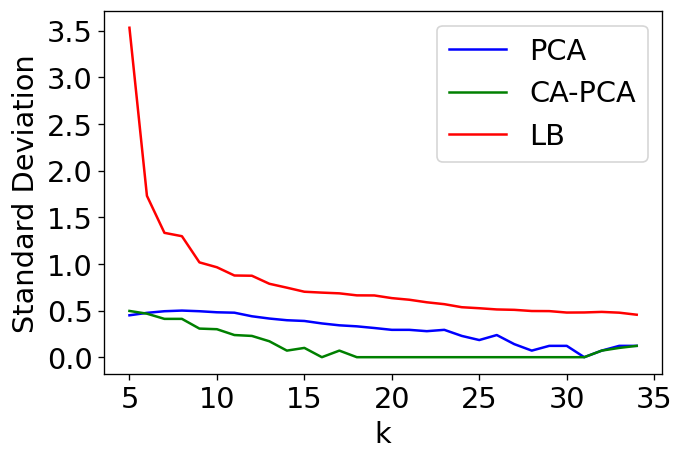}
         \caption{Standard deviation of estimated ID for Klein bottle}
         \label{fig:Klein std dev}
     \end{subfigure}
        \caption{Analysis of Non-Manifolds}
        \label{fig:non-manifold}
\end{figure}

Thus, CA-PCA, plus PCA and perhaps LB, may potentially assist in determining if data comes from a manifold.

\subsection{Robustness to Error}

Lastly, we test our estimator for robustness by introducing error into the samples. Error was added to each point by independently adding to each coordinate a random number sampled uniformly from the interval $(-\epsilon,\epsilon)$ for some choice of $\epsilon$. Figure \ref{fig:Klein error} shows the results of the tests when this error was introduced to 400 points on the Klein bottle with $\epsilon=1$ (compare the results to those in Figure \ref{fig:Klein} and $\epsilon$ to the choice of $a=10,b=5$).

The same method was applied to 20,000 points from $SO(5)$ with $\epsilon=.1$, with the results depicted in Figure \ref{fig:SO(5) error}. In both cases of error, CA-PCA's estimates are slightly further away from the true ID, yet still close and competitive with those from PCA. (Compare the results to those of Figure \ref{fig:SO(5)}.) This comes as no surprise for small levels of error, given that they lead to small changes in the eigenvalues computed from principal component analysis, which should lead to only the occasional change in the minimization problems \eqref{eq:PCA_method} and \eqref{eq:error to be minimized}.

Lastly, we run the three methods on the Klein bottle, fixing the value of $k=20$ and vary $\epsilon$ from 0 to 1.4 with step sizes of 0.1. Similar to before, we sample $n=400$ points from the Klein bottle and run 200 trials per instance of $\epsilon$. Results are displayed in Figure \ref{fig:wasup}. As one can see, all methods are (at least somewhat) robust up to a certain level of error. Once the error increases to a certain level, the thickened Klein bottle in which our sample (with error) lies may be considered a 4-dimensional manifold. In particular, the ball containing the nearest neighbors of the base point may be fully contained in this 4-dimensional manifold, seeing points spread equally in all directions.

\begin{figure}[hbt!]
    \centering
    \begin{subfigure}[b]{0.3\columnwidth}
         \centering
         \includegraphics[width=\columnwidth]{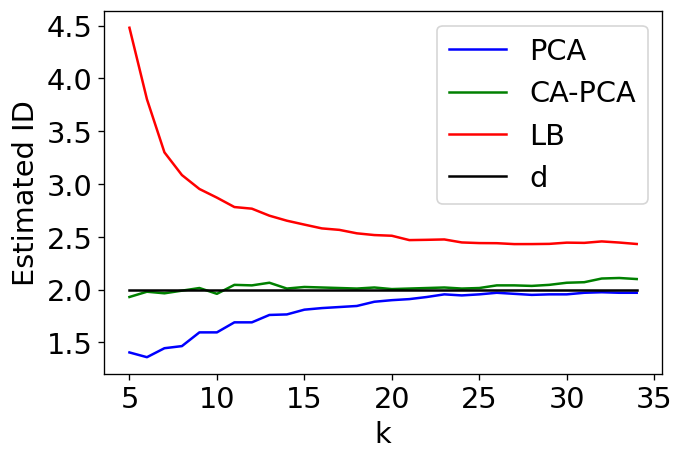}
         \caption{Estimated ID for Klein bottle with error, varying $k$}
         \label{fig:Klein error}
     \end{subfigure}
     \hfill
     \begin{subfigure}[b]{0.3\columnwidth}
         \centering
         \includegraphics[width=\columnwidth]{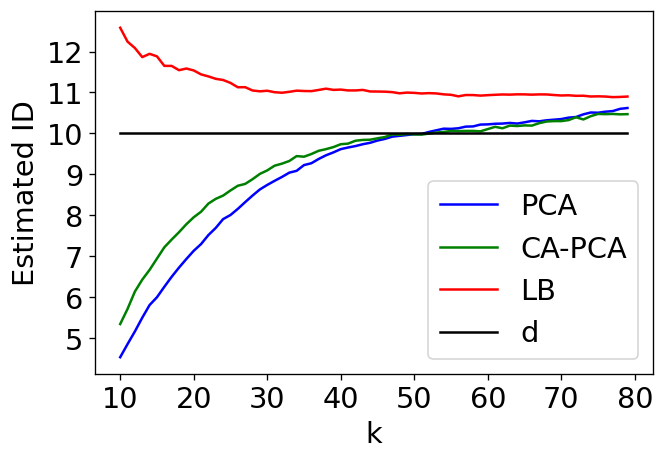}
         \caption{Estimated ID for $SO(5)$ with error}
         \label{fig:SO(5) error}
     \end{subfigure}
    \hfill
    \begin{subfigure}[b]{0.3\columnwidth}
         \centering
         \includegraphics[width=\columnwidth]{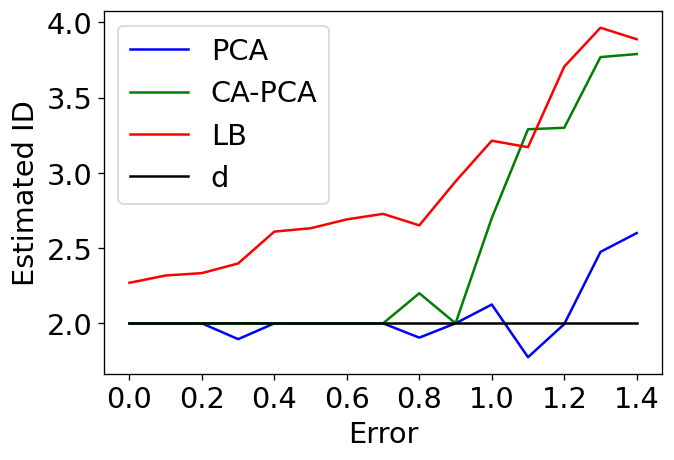}
         \caption{Estimated ID for Klein bottle with error, varying error}
         \label{fig:wasup}
     \end{subfigure}
     \caption{Synthetic Manifolds with Error}
     \label{fig:synthetic error}
\end{figure}

\subsection{Varying Curvature and Sample Size}

To better demonstrate the relevance of curvature in our comparison between CA-PCA and PCA, we run the three methods on the surface parameterized by
\begin{equation}\label{eq:fancy_surface}
    (u,v)\mapsto (u,v,a(u^2+v^2),a(u^2-v^2)), 0\le u,v\le 1,
\end{equation}
varying $a$ from 0 to 7.25 with step size 0.25 as a means of increasing the curvature. We sample 400 points at random from this surface with respect to the uniform measure on $(u,v)\in[-1,1]^2$. While this is not exactly the uniform measure on the underlying surface, we again appeal to the smooth density.

Next, we take the $k=15$ nearest neighbors of the origin in $\R^4$ for each of the three tests. This serves two purposes. First, it avoids any issues of the boundary of the manifold. Second, it ensures that the parameter $a$ reflects the magnitude of the curvature \textit{at the sampled points}. Since each trial uses the same $k$ points, we resample 400 points from the surface for each of $N=200$ trials. We then repeat this process for each value of $a$.

Results are plotted in Figure \ref{fig:curvature}. Both PCA and CA-PCA remain roughly constant for $0\le a\le 4$, though PCA begins to increase around $a=4$ and severely overestimates the dimension around $5\le a\le 7$. CA-PCA, on the other hand, remains more steady before beginning to underestimate the dimension around $a=6$ (though by not as much). This may be due to violation of our small curvature assumption. Interestingly enough, CA-PCA seems to do better than PCA even in the case of $a=0$ when the surface is a flat vector subspace. More investigation is needed to determine the cause of this phenomenon.


\begin{figure}[hbt!]
    \centering
    \begin{subfigure}[t]{0.45\columnwidth}
         \centering
         \includegraphics[width=\columnwidth]{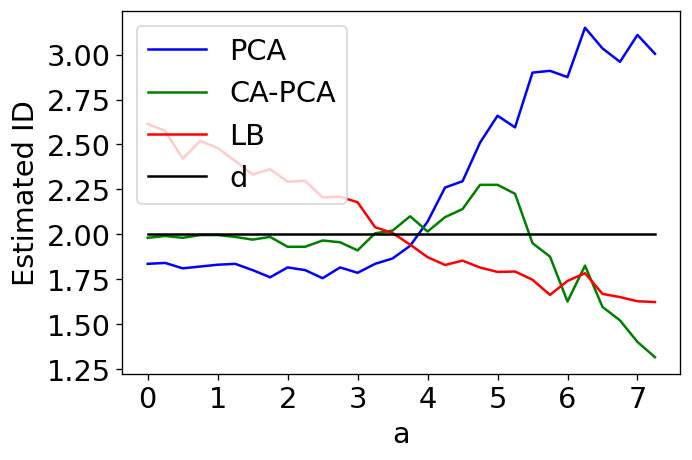}
         \caption{Estimated ID for surface in \eqref{eq:fancy_surface}, $a$ varying}
         \label{fig:curvature}
     \end{subfigure}
     \hfill
     \begin{subfigure}[t]{0.45\columnwidth}
         \centering
         \includegraphics[width=\columnwidth]{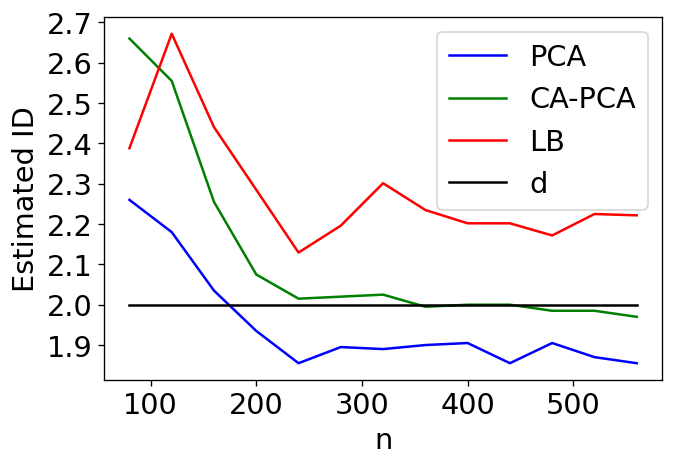}
         \caption{Estimated ID for the Klein bottle, $n$ varying}
         \label{fig:sample_size}
     \end{subfigure}
     \caption{Varying Parameters}
        \label{fig:vary}
\end{figure}

Lastly, we investigate the effect of the number of points $n$ sampled from the manifold on the methods. While this is not a direct input of Algorithm \ref{alg:CAPCA}, it is still something which the user may have some choice over. It may be necessary to balance the success of our method with the (computational) cost of the nearest neighbors method and the (computational and/or physical) cost of obtaining more data.

Furthermore, varying $n$ amounts to varying the effects of curvature in the following sense. Suppose $k$ is fixed. Then, if $n$ is significantly decreased, one must go further ``around the manifold'' from a point to obtain its $k$ nearest neighbors, putting one at greater risks for high levels of nonlinearity.

In testing, we sampled $n$ points from the Klein bottle, $n$ varying from 80 to 560, and ran the three methods $N=200$ times on each sample with a value of $k=20$. At the lowest values of $n$, when the 20 nearest neighbors are likely to wrap furthest around the Klein bottle, CA-PCA actually struggles more than PCA. This may be due to a violation of our small curvature assumption. However, as soon as 20 points is a mere 10\% of the sample ($n=200$), CA-PCA does about as well as one could expect and begins to outperform PCA again.

\section{Conclusion}\label{sec:conclusion}

As we have seen, the use of local PCA as an ID estimator may be improved in a variety of settings by taking into account the curvature of the underlying manifold while generally providing good estimates and competing with LB.

CA-PCA is merely one of many potential applications of curvature analysis to ID estimators. In particular, calculations in the Appendix (\eqref{eq:surface_area_calc} with variable radius) show that the $d$-dimensional volume of a ball of radius $r$ intersected with a $d$-dimensional manifold $\mathcal{M}$ is proportional to $r^d+cr^{d+2}$, where $c$ is a constant depending on $\mathcal{M}$ (see Subsection \ref{subsec:volume growth}). There are many ID estimators derived using a volume proportional to $r^d$ or related notions like average distances from the center of a ball. It would be interesting to see if any of these tests could be improved by taking curvature into account.

Another place for potential further study is that of ID estimation for manifolds with boundary. In our experiments, CA-PCA gets closer to the ``true'' ID of manifolds with boundary than PCA or LB, though more analysis is needed to determine if this is more than a coincidence.

Lastly, we hope that future work will lead to a process for determining a single value of $k$, or at least a few values of $k$ from which one may confidently estimate the dimension. Our work suggests that for manifolds of low dimension, choosing $10\le k\le 25$ in CA-PCA may suffice, while manifolds of greater dimension may require $30\le k\le 50$. Choosing even larger $k$ does not seem to hurt, except in the case of the curve \eqref{eq:fancy_curve} and the airplane photos, and both ranges seem lower than the corresponding ranges for PCA. However, more study and perhaps more specifically designed experiments may be required.

\newpage
\appendix
\onecolumn






\section{Integral Computations}\label{sec:appendix}

\subsection{Some Elementary Computations}

Let $\sigma(x)$ denote the (non-normalized) surface measure on the sphere induced by Lebesgue measure on $\R^d$. We will often write $\sigma(x)$ when $x\in \R^d$ and $\sigma(\theta)$ when $\theta\in S^{d-1}$, the unit sphere in $\R^d$.

The following elementary fact can be found in \cite{Folland}, for instance.
\begin{theorem}\label{thm:sphere integrals}
Let $P(x)=x_1^{\alpha_1}\cdots x_d^{\alpha_d}$ for $\alpha_1,...,\alpha_d\in\{0,1,2,...\}$. Write $\beta_j=\frac{1}{2}(\alpha_j+1)$. Then, if all $\alpha_j$ are even,
\begin{equation*}
    \int_{S^{d-1}} P(x)d\sigma(x)=\frac{2\Gamma(\beta_1)\cdots\Gamma(\beta_d)}{\Gamma(\beta_1+...+\beta_d)}.
\end{equation*}

If any $\alpha_j$ are odd, then the above integral is zero.
\end{theorem}

For our purposes, we will need the values $\Gamma(1/2)=\sqrt{\pi}, \Gamma(3/2)=\frac{\sqrt{\pi}}{2},\Gamma(5/2)=\frac{3\sqrt{\pi}}{4},$ and $\Gamma(7/2)=\frac{15\sqrt{\pi}}{8}$ and the identity $\Gamma(x+1)=x\Gamma(x)$.

As particular instances of Theorem \ref{thm:sphere integrals},
\begin{equation}\label{eq:sphere 0}
    \int_{S^{d-1}}d\sigma(x)=\frac{2\pi^{d/2}}{\Gamma(d/2)}
\end{equation}

\begin{equation}\label{eq:sphere 2}
    \int_{S^{d-1}}x_i^2d\sigma(x)=\frac{\pi^{d/2}}{\Gamma(d/2+1)}
\end{equation}

\begin{equation}\label{eq:sphere 2-2}
    \int_{S^{d-1}}x_i^2x_j^2d\sigma(x)=\frac{\pi^{d/2}}{2\Gamma(d/2+2)}
\end{equation}

\begin{equation}\label{eq:sphere 4}
    \int_{S^{d-1}}x_i^4d\sigma(x)=\frac{3\pi^{d/2}}{2\Gamma(d/2+2)}
\end{equation}

\begin{equation}\label{eq:sphere 2-2-2}
    \int_{S^{d-1}}x_i^2x_j^2x_k^2d\sigma(x)=\frac{\pi^{d/2}}{4\Gamma(d/2+3)}
\end{equation}

\begin{equation}\label{eq:sphere 4-2}
    \int_{S^{d-1}}x_i^4x_j^2d\sigma(x)=\frac{3\pi^{d/2}}{4\Gamma(d/2+3)}
\end{equation}

\begin{equation}\label{eq:sphere 6}
    \int_{S^{d-1}}x_i^6d\sigma(x)=\frac{15\pi^{d/2}}{4\Gamma(d/2+3)}
\end{equation}

\begin{lemma}\label{lemma:sphere quad}
Let $Q:\R^d\to\R$ be a quadratic form with eigenvalues $\mu_1,...,\mu_d$. Then,
\begin{equation}\label{eq:sphere quad}
     \int_{S^{d-1}}Q(\theta)d\theta=\frac{\pi^{d/2}}{\Gamma(d/2+1)}\sum_{k}\mu_{k}
\end{equation}

Also,
\begin{equation}\label{eq:sphere quad^2}
    \int_{S^{d-1}}Q(\theta)^2d\theta=\left(\frac{1}{2}\left(\sum_{k=1}^d\mu_{k}\right)^2+\sum_{k=1}^d\mu_{k}^2\right)\frac{\pi^{d/2}}{\Gamma(d/2+2)}.
\end{equation}

\end{lemma}

\begin{proof}
To establish \eqref{eq:sphere quad}, we rotate coordinates so the eigenvectors of $Q$ are $e_1,...,e_d$ and use \eqref{eq:sphere 2}:
\begin{align*}
    \int_{S^{d-1}}{Q}(\theta)d\theta&=\int_{S^{d-1}}\sum_{k=1}^d\mu_{k}x_k^2d\theta\\
    &=\sum_{k=1}^d\mu_{k}\int_{S^{d-1}}x_k^2d\theta\\
    &=\frac{\sqrt{\pi}^d}{\Gamma(d/2+1)}\sum_{k=1}^d\mu_{k}.
\end{align*}

For \eqref{eq:sphere quad^2}, we repeat the same rotation of coordinates and use \eqref{eq:sphere 2-2} and \eqref{eq:sphere 4}:
\begin{align*}
    \int_{S^{d-1}}Q(\theta)^2d\theta&=\int_{S^{d-1}}\left(\sum_{k=1}^d\mu_{k}x_k^2\right)^2 d\theta\\
    &=\sum_{1\le k,l\le d}\mu_{k}\mu_{l}\int_{S^{d-1}}x_k^2x_l^2d\theta\\
    &=\sum_{k=1}^d\mu_{k}^2\left[\frac{3}{2}\frac{\sqrt{\pi}^d}{\Gamma(d/2+2)}\right]+\sum_{k\ne l}\mu_{k}\mu_{l}\left[\frac{1}{2}\frac{\sqrt{\pi}^d}{\Gamma(d/2+2)}\right]\\
    &=\left(\frac{1}{2}(\sum_k\mu_{k})^2+\sum_k\mu_
    {k}^2\right)\frac{\sqrt{\pi}^d}{\Gamma(d/2+2)}
\end{align*}

\end{proof}

\subsection{The Region}

First, let
\begin{equation*}
    q(x)=\sum_{j=1}^{D-d}Q_j(x)^2
\end{equation*}
and note $q(x)=O(\Lambda^2)$ for $|x|\le 1$.

Writing $x=r\theta$, with $r\ge0$ and $\theta\in S^{d-1}$, let $r=r(\theta)$ denote the positive real solution to
\begin{equation*}
    r^2+q(r\theta)=r^2+r^4q(\theta)=1.
\end{equation*}

By the quadratic formula and Taylor expansion $\sqrt{1+t}=1+\frac{t}{2}+O(t^2)$,
\begin{align*}
    r^2(\theta)&=\frac{-1+\sqrt{1+4q(\theta)}}{2q(\theta)}\\
    &=\frac{-1+1+2q(\theta)-2q(\theta)^2+O(q(\theta)^3)}{2q(\theta)}\\
    &=1-q(\theta)+O(\Lambda^4).
\end{align*}

Thus,
\begin{equation}\label{eq:r(theta)}
    r(\theta)=1-\frac{q(\theta)}{2}+O(\Lambda^4).
\end{equation}
and
\begin{equation}\label{eq:r(theta)^m}
    r(\theta)^m=1-\frac{mq(\theta)}{2}+O(\Lambda^4).
\end{equation}

This approximation will be used throughout this section as we integrate over spherical coordinates $(r,\theta)$ and need to find the limits of our integral in $r$.

\subsection{Total Surface Area}\label{subsec:volume growth}

To begin the computation of the volume of $\mathcal{M}\cap B_1(0)$, we switch to polar coordinates, substitute the expression of $S(x)$ in \eqref{eq:form for S(x)}, and integrate with respect to the radius
\begin{align*}
    \int_R S(x) dx&=\int_{S^{d-1}}\int_0^{r(\theta)}S(r\theta)r^{d-1}drd\theta\\
    &=\int_{S^{d-1}}\int_0^{r(\theta)}\left(1+2\sum_{j=1}^{D-d}\tilde{Q}_j(r\theta)\right)r^{d-1}drd\theta +O(\Lambda^4)\\
    &=\int_{S^{d-1}}\int_0^{r(\theta)}r^{d-1}+2r^{d+1}\sum_{j=1}^{D-d}\tilde{Q}_j(\theta)drd\theta\\
    &=\int_{S^{d-1}}\frac{1}{d}r(\theta)^d+\frac{2}{d+2}r(\theta)^{d+2}\sum_{j=1}^{D-d}\tilde{Q}_j(\theta)d\theta.
\end{align*}

By \eqref{eq:r(theta)^m},
\begin{align*}
    \int_R S(x) dx&=\int_{S^{d-1}}\frac{1}{d}(1-\frac{dq(\theta)}{2})+\frac{2}{d+2}(1-\frac{(d+2)q(\theta)}{2})\sum_{j=1}^{D-d}\tilde{Q}_j(\theta)d\theta+O(\Lambda^4)\\
    &=\int_{S^{d-1}}\frac{1}{d}-\frac{q(\theta)}{2}+\frac{2}{d+2}\sum_{j=1}^{D-d}\tilde{Q}_j(\theta)d\theta+O(\Lambda^4)
\end{align*}

By \eqref{eq:sphere 0}, \eqref{eq:sphere quad}, and \eqref{eq:sphere quad^2}, we integrate over $S^{d-1}$ and find
\begin{align}\label{eq:surface_area_calc}
    \int_RS(x)dx=&\frac{2\pi^{d/2}}{d\Gamma(d/2)}-\frac{\pi^{d/2}}{2\Gamma(d/2+2)}\sum_j\left(\frac{1}{2}(\sum_i\lambda_{i,j})^2+
    \sum_i\lambda_{i,j}^2\right)\\&+\frac{2\pi^{d/2}}{(d+2)\Gamma(d/2+1)}\sum_{i,j}\lambda_{i,j}^2+O(\Lambda^4)\\
    =&\frac{\pi^{d/2}}{\Gamma(d/2)}\left[\frac{2}{d}-\frac{2}{d(d+2)}\left(A+\frac{1}{2}B\right)+\frac{4}{d(d+2)}A\right]+O(\Lambda^4)\\
    =&\frac{\pi^{d/2}}{\Gamma(d/2)}\left[\frac{2}{d}-\frac{1}{d(d+2)}B+\frac{2}{d(d+2)}A\right]+O(\Lambda^4).
\end{align}

\subsection{Normalizing with Respect to Surface Area}

In this subsection, we take a quick detour to show how Taylor series expansion will be used to divide by $\int_R S(x)dx$ in normalizing the relevant integrals.

We use the expansion
\begin{align}\label{eq:fancy expansion}
    \frac{a_0+a_1t+O(t^2)}{b_0+b_1t+O(t^2)}&=\left(\frac{a_0}{b_0}+\frac{a_1}{b_0}t+O(t^2)\right)\frac{1}{1+\frac{b_1}{b_0}t+O(t^2)}\\
    &=\left(\frac{a_0}{b_0}+\frac{a_1}{b_0}t+O(t^2)\right)\left(1-\frac{b_1}{b_0}t+O(t^2)\right)\\
    &=\frac{a_0}{b_0}+\frac{a_1}{b_0}t-\frac{a_0b_1}{b_0^2}t+O(t^2)
\end{align}
to estimate these ratios (matching $t$ with $\Lambda^2$, roughly)

\subsection{The Means}

By Lemma \ref{lemma:form for S(x)} and integrating in the radial direction,

\begin{align*}
    \int_R Q_i(x)S(x)dx&=\int_{S^{d-1}}\int_0^{r(\theta)}Q_i(x)\left[r^{d-1}+2r^{d+1}\sum_{j=1}^{D-d}\tilde{Q}_j(\theta)\right]\left(1+2\sum_{j=1}^{D-d}\tilde{Q}_j(r\theta)\right)drd\theta \\&+O(\Lambda^4)\\
    =&\int_{S^{d-1}}\int_0^{r(\theta)}Q_i(x)r^{d-1}drd\theta + O(\Lambda^3)\\
    =&\int_{S^{d-1}}\int_0^{r(\theta)}r^{d+1}Q_i(\theta)drd\theta+ O(\Lambda^3)\\
    =&\int_{S^{d-1}}\frac{1}{d+2}Q_i(\theta)r(\theta)^{d+2}d\theta+ O(\Lambda^3).
\end{align*}    
    
Thus, applying \eqref{eq:r(theta)^m} and \eqref{eq:sphere quad},
\begin{align*}
    \int_R Q_i(x)S(x)dx&=\int_{S^{d-1}}\frac{1}{d+2}Q_i(\theta)d\theta+ O(\Lambda^3)\\
    &=\frac{\pi^{d/2}}{(d+2)\Gamma(d/2+1)}\sum_{j=1}^d\lambda_{i,j}+O(\Lambda^3).
\end{align*}

Upon normalizing, we find
\begin{equation*}
    \bar{Q}_i=\frac{\int_R Q_i(x)S(x)dx}{\int_R S(x)dx}=\frac{1}{d+2}\sum_{j=1}^d\lambda_{i,j}+O(\Lambda^2).
\end{equation*}

Note that we only need this computation up to first-order accuracy since we will only need the above expression to multiply it by other terms depending on $\Lambda$ in the computation of $I_4$.

\subsection{Upper Trace}

Let $\theta^*\in S^{d-1}$ be arbitrary. We begin as with the computation of total surface area, converting to polar coordinates and integrating in $r$ first to get
\begin{align*}
    \int_R (x\cdot \theta^*)^2S(x)dx&=\int_{S^{d-1}}\int_0^{r(\theta)}(r\theta\cdot \theta^*)^2\left[r^{d-1}+2r^{d+1}\sum_{j=1}^{D-d}\tilde{Q}_j(\theta)\right]drd\theta+O(\Lambda^4)\\
    &=\int_{S^{d-1}}(\theta\cdot \theta^*)^2\int_0^{r(\theta)}r^{d+1}+2r^{d+3}\sum_{j=1}^{D-d}\tilde{Q}_j(\theta)drd\theta+O(\Lambda^4)\\
    &=\int_{S^{d-1}}(\theta\cdot \theta^*)^2\left[\frac{1}{d+2}r(\theta)^{d+2}+\frac{2}{d+4}r(\theta)^{d+4}\sum_{j=1}^{D-d}\tilde{Q}_j(\theta)\right]d\theta+O(\Lambda^4)
\end{align*}

By \eqref{eq:r(theta)^m}, may rewrite the above as
\begin{equation}
    \int_{S^{d-1}}(\theta\cdot \theta^*)^2\left[\frac{1}{d+2}(1-\frac{(d+2)q(\theta)}{2})+\frac{2}{d+4}(1-\frac{(d+4)q(\theta)}{2})\sum_{j=1}^{D-d}\tilde{Q}_j(\theta)\right]d\theta+O(\Lambda^4),\label{eq:I1 with theta^*}
\end{equation}
or
\begin{equation}
    \int_R (x\cdot \theta^*)^2S(x)dx=\int_{S^{d-1}}(\theta\cdot \theta^*)^2\left[\frac{1}{d+2}-\frac{q(\theta)}{2}+\frac{2}{d+4}\sum_{j=1}^{D-d}\tilde{Q}_j(\theta)\right]d\theta+O(\Lambda^4)\label{eq:I1 with theta^* II}
\end{equation}

By taking $\theta^*=e_i$ ($1\le i\le d$) and summing over $i$, we obtain the sum of the first $d$ eigenvalues. By the identity $\sum_{i=1}^d (\theta\cdot e_i)^2=1$ for $\theta\in S^{d-1}$, we have
\begin{align*}
    \sum_i (\Sigma_1)_{ii}=&\int_{S^{d-1}}\left[\frac{1}{d+2}-\frac{q(\theta)}{2}+\frac{2}{d+4}\sum_{j=1}^{D-d}\tilde{Q}_j(\theta)\right]d\theta+O(\Lambda^4)\\
    =&\frac{2\pi^{d/2}}{(d+2)\Gamma(d/2)}-\frac{\pi^{d/2}}{2\Gamma(d/2+2)}\sum_j\left(\frac{1}{2}(\sum_k\lambda_{k,j})^2+\sum_k\lambda_{k,j}^2\right)\\&+\frac{2\pi^{d/2}}{(d+4)\Gamma(d/2+1)}\sum_{j,k}\lambda_{k,j}^2+O(\Lambda^4)\\
    =&\frac{\pi^{d/2}}{\Gamma(d/2)}\left[\frac{2}{d+2}-\frac{1}{d(d+2)}B+\frac{2}{(d+2)(d+4)}A\right]+O(\Lambda^4).
\end{align*}

To complete our computation of the upper trace, we divide by the total surface area

\begin{align*}
    \sum_i \frac{(\Sigma_1)_{ii}}{\int_R S(x)dx}=&\frac{\frac{\pi^{d/2}}{\Gamma(d/2)}\left[\frac{2}{d+2}-\frac{1}{d(d+2)}B+\frac{2}{(d+2)(d+4)}A\right]+O(\Lambda^4)}{\frac{\pi^{d/2}}{\Gamma(d/2)}\left[\frac{2}{d}-\frac{1}{d(d+2)}B+\frac{2}{d(d+2)}A\right]+O(\Lambda^4)}\\
    =&\frac{d}{d+2}-\frac{1}{2(d+2)}B+\frac{d}{(d+2)(d+4)}A+\frac{d}{2(d+2)^2}B-\frac{d}{(d+2)^2}A\\&+O(\Lambda^4)\\
    =&\frac{d}{d+2}-\frac{2d}{(d+2)^2(d+4)}A-\frac{1}{(d+2)^2}B+O(\Lambda^4).
\end{align*}

\subsection{Bounds for Upper Eigenvalues}

Here, we will determine the lower and upper bounds for the eigenvalues of $\Sigma_1$ (found in \eqref{eq:upper bound for lambda_i} and \eqref{eq:lower bound for lambda_i} by approximating the expression in \eqref{eq:I1 with theta^* II} without particular choice of $\theta^*$.

Fix $1\le j\le D-d$ and perform a rotation so the eigenvectors of $\tilde{Q}_j$ become $e_1,...,e_d$ (keeping $\theta^*=(y_1,...,y_d)$ arbitrary). By ignoring any terms with odd degrees of $x_k$ and applying \eqref{eq:sphere 2-2} and \eqref{eq:sphere 4} 
\begin{align} 
    \int_{S^{d-1}}(\theta\cdot y)^2\tilde{Q}_j(\theta)d\theta&=\sum_{k=1}^dy_k^2\lambda_{k,j}^2\int_{S^{d-1}}x_k^4dx+\sum_{k\ne l}y_k^2\lambda_{l,j}^2\int_{S^{d-1}}x_k^2x_l^2dx\\
    &=\sum_{k}y_k^2\lambda_{k,j}^2\left(\frac{3}{2}\frac{\pi^{d/2}}{\Gamma(d/2+2)}\right)+\sum_{k\ne l}y_k^2\lambda_{l,j}^2\left(\frac{1}{2}\frac{\pi^{d/2}}{\Gamma(d/2+2)}\right)\\
    &\le \sum_k y^2_k \sum_l \lambda_{l,j}^2\left(\frac{3}{2}\frac{\pi^{d/2}}{\Gamma(d/2+2)}\right)\\
    &=\frac{3}{2}\frac{\pi^{d/2}}{\Gamma(d/2+2)}\sum_k \lambda_{k,j}^2 \label{eq:upper bound Q-tilde}
\end{align}

Similarly,
\begin{equation}\label{eq:lower bound Q-tilde}
    \int_{S^{d-1}}(\theta\cdot y)^2\tilde{Q}_j(\theta)d\theta\ge \frac{1}{2}\frac{\pi^{d/2}}{\Gamma(d/2+2)}\sum_k \lambda_{k,j}^2.
\end{equation}

Repeating the above strategy, we bound 
\begin{equation*}
    \int_{S^{d-1}}(\theta\cdot y)^2Q_j(\theta)^2d\theta=\int_{S^{d-1}}\left(\sum_{k=1}^dy_k^2x_k^2\right)\left(\sum_{l=1}^d\lambda_{l,j}x_l^2\right)\left(\sum_{m=1}^d\lambda_{m,j}x_m^2\right)d\theta.
\end{equation*}
Expanding the above and applying \eqref{eq:sphere 2-2-2}, \eqref{eq:sphere 4-2}, and \eqref{eq:sphere 6} gives
\begin{equation*}
    \frac{\pi^{d/2}}{\Gamma(d/2+3)}\left(\sum_{k}\frac{15y_k^2\lambda_{k,j}^2}{4}+\sum_{k\ne l\ne m\ne k}\frac{y_k^2\lambda_{j,l}\lambda_{m,j}}{4}+2\sum_{k\ne l}\frac{3y_k^2\lambda_{k,j}\lambda_{l,j}}{4}+\sum_{k\ne l}\frac{3y_k^2\lambda_{l,j}^2}{4}\right).
\end{equation*}
    
Thus,     
\begin{align}
    \int_{S^{d-1}}(\theta\cdot y)^2Q_j(\theta)^2d\theta&\le \frac{\pi^{d/2}}{\Gamma(d/2+3)}\left(\sum_{k}\frac{15y_k^2}{4}\sum_l\lambda_{l,j}^2+\sum_ky_k^2\frac{3}{2}\left((\sum_{l}\lambda_{l,j})^2+\sum_l\lambda_{l,j}^2\right)\right)\\
    &=\frac{\pi^{d/2}}{\Gamma(d/2+3)}\left(\frac{21}{4}\sum_l\lambda_{l,j}^2+\frac{3}{2}(\sum_{l}\lambda_{l,j})^2\right)\label{eq:upper bound Q_j}
\end{align}

Here, we simply acknowledge
\begin{equation}\label{eq:lower bound Q_j}
    \int_{S^{d-1}}(\theta\cdot y)^2Q_j(\theta)^2d\theta\ge0.
\end{equation}

Subbing \eqref{eq:upper bound Q-tilde} and \eqref{eq:lower bound Q_j} into \eqref{eq:I1 with theta^* II}, we find
\begin{align*}
    \int_R (x\cdot\theta^*)^2S(x) dx&\le\int_{S^{d-1}}(\theta\cdot \theta^*)^2\frac{1}{d+2}d\theta+\frac{2}{d+4}\frac{3}{2}\frac{\pi^{d/2}}{\Gamma(d/2+2)}\sum_{k,j} \lambda_{k,j}^2+O(\Lambda^4)\\
    &=\int_{S^{d-1}}x_1^2\frac{1}{d+2}d\theta+\frac{3}{d+4}\frac{\pi^{d/2}}{\Gamma(d/2+2)}A+O(\Lambda^4)\\
    &=\frac{1}{d+2}\frac{\pi^{d/2}}{\Gamma(d/2+1)}+\frac{3}{d+4}\frac{\pi^{d/2}}{\Gamma(d/2+2)}A+O(\Lambda^4)\\
    &=\frac{\pi^{d/2}}{\Gamma(d/2)}\left[\frac{2}{d(d+2)}+\frac{12}{d(d+2)(d+4)}A\right]+O(\Lambda^4)
\end{align*}

To normalize the upper bound we divide the above by $\int_R S(x)dx$ to get
\begin{align*}
    \frac{\int_R(x\cdot\theta^*)^2S(x)dx}{\int_RS(x)dx}&\le\frac{\frac{\pi^{d/2}}{\Gamma(d/2)}\left[\frac{2}{d(d+2)}+\frac{12}{d(d+2)(d+4)}A\right]+O(\Lambda^4)}{\frac{\pi^{d/2}}{\Gamma(d/2)}\left[\frac{2}{d}-\frac{1}{d(d+2)}B+\frac{2}{d(d+2)}A\right]+O(\Lambda^4)}\\
    &=\frac{1}{d+2}+\frac{6}{(d+2)(d+4)}A\\&-\frac{1}{d+2}\left[-\frac{1}{2(d+2)}B+\frac{1}{(d+2)}A\right]+O(\Lambda^4)\\
    &=\frac{1}{d+2}+\frac{5d+8}{(d+2)^2(d+4)}A+\frac{1}{2(d+2)^2}B+O(\Lambda^4)
\end{align*}

Subbing \eqref{eq:lower bound Q-tilde} and \eqref{eq:upper bound Q_j} into \eqref{eq:I1 with theta^* II}, we see
\begin{align*}
    \int_R (x\cdot\theta^*)^2S(x) dx&\ge\frac{1}{d+2}\frac{\pi^{d/2}}{\Gamma(d/2+1)}-\frac{\pi^{d/2}}{\Gamma(d/2+3)}\left(\frac{21}{4}\sum_j\sum_l\lambda_{l,j}^2+\frac{3}{2}\sum_j(\sum_{l}\lambda_{l,j})^2\right)\\
    &+\frac{1}{d+4}\frac{\pi^{d/2}}{\Gamma(d/2+2)}\sum_j\sum_k \lambda_{k,j}^2+O(\Lambda^4)\\
    &=\frac{1}{d+2}\frac{\pi^{d/2}}{\Gamma(d/2+1)}-\frac{\pi^{d/2}}{\Gamma(d/2+3)}\left(\frac{21}{4}A+\frac{3}{2}B\right)\\
    &+\frac{\pi^{d/2}}{\Gamma(d/2+3)}\frac{1}{2}A+O(\Lambda^4)\\
    &=\frac{\pi^{d/2}}{\Gamma(d/2)}\left[\frac{2}{d(d+2)}-\frac{8}{d(d+2)(d+4)}\left(\frac{19}{4}A+\frac{3}{2}B\right)\right]+O(\Lambda^4).
\end{align*}

To obtain the lower bound on the eigenvalues, we divide by total surface area:
\begin{align*}
    \frac{\int_R (x\cdot\theta^*)^2S(x) dx}{\int_R S(x) dx}\ge& \frac{\frac{2}{d(d+2)}-\frac{8}{d(d+2)(d+4)}\left(\frac{19}{4}A+\frac{3}{2}B\right)+O(\Lambda^4)}{\frac{2}{d}-\frac{1}{d(d+2)}B+\frac{2}{d(d+2)}A+O(\Lambda^4)}\\
    =&\frac{1}{d+2}-\frac{19}{(d+2)(d+4)}A-\frac{6}{(d+2)(d+4)}B+\frac{1}{2(d+2)^2)}B\\
    &-\frac{1}{(d+2)^2}A+O(\Lambda^4)\\
    =&\frac{1}{d+2}-\frac{20d+42}{(d+2)^2(d+4)}A-\frac{11d+20}{2(d+2)^2(d+4)}B+O(\Lambda^4)
\end{align*}

\subsection{Lower Eigenvalues}

We repeat the strategy of prior computations, taking advantage of the fact that $Q_j(x)^2-\bar{Q}_j^2=O(\Lambda^2)$ so the $O(\Lambda^2)$ terms in $S(x)$ turn into higher-order error. Similarly, we are able to take $r(\theta)\approx 1$ once we have integrated in $r$.

\begin{align*}
    \int_R (Q_j(x)^2-\bar{Q}_j^2)S(x)dx&=\int_R Q_j(x)^2-\bar{Q}_j^2dx + O(\Lambda^4)\\
    &=\int_{S^{d-1}}\int_0^{r(\theta)}r^{d-1}\left[r^2Q_j(\theta)^2-\bar{Q}_j^2\right]drd\theta+O(\Lambda^4)\\
    &=\int_{S^{d-1}}\frac{1}{d+2}r(\theta)^{d+2}Q_j(\theta)^2-\frac{1}{d}r(\theta)^d\bar{Q}_j^2 d\theta+O(\Lambda^4)\\
    &=\int_{S^{d-1}}\frac{1}{d+2}Q_j(\theta)^2-\frac{1}{d}\bar{Q}_j^2 d\theta+O(\Lambda^4)\\
\end{align*}

By Lemma \ref{lemma:sphere quad} and substituting our value of $\bar{Q}_i$,
\begin{align}
    \int_R (Q_j(x)^2-\bar{Q}_j^2)S(x)dx=&\left(\frac{1}{2(d+2)}(\sum_k\lambda_{k,j})^2+\frac{1}{d+2}\sum_k\lambda_{k,j}^2\right)\frac{\pi^{d/2}}{\Gamma(d/2+2)}\\&-\frac{1}{d}\frac{2\pi^{d/2}}{\Gamma(d/2)}\left(\bar{Q}_i\right)^2+O(\Lambda^3)\\
    =&\left(\frac{1}{2(d+2)}(\sum_k\lambda_{k,j})^2+\frac{1}{d+2}\sum_k\lambda_{k,j}^2\right)\frac{\pi^{d/2}}{\Gamma(d/2+2)}\\
    &-\frac{1}{d}\frac{2\pi^{d/2}}{\Gamma(d/2)}\left(\frac{1}{d+2}\sum_{j=1}^d\lambda_{k,j}+O(\Lambda^2)\right)^2+O(\Lambda^3)\label{eq:for2xfactor}\\
    =&\frac{\pi^{d/2}}{\Gamma(d/2)}\left[\frac{2}{d(d+2)^2}B_j+\frac{4}{d(d+2)^2}A_j-\frac{2}{d(d+2)^2}B_j\right]+O(\Lambda^3)\\
    =&\frac{\pi^{d/2}}{\Gamma(d/2)}\frac{4}{d(d+2)^2}A_j+O(\Lambda^3).
\end{align}

We normalize to get
\begin{align*}
    \frac{\int_R (Q_i(x)-\bar{Q}_i)^2S(x)dx}{\int_R S(x)dx}&=\frac{\frac{4}{d(d+2)^2}A_j+O(\Lambda^3)}{\frac{2}{d}-\frac{1}{d(d+2)}B+\frac{2}{d(d+2)}A+O(\Lambda^4)}\\
    &=\frac{2}{(d+2)^2}A_j+O(\Lambda^3).
\end{align*}

Note that the coordinate transformation on $\R^d\times \R^{D-d}$ by $(x,y)\mapsto(-x,y)$ applied to the manifold $\mathcal{M}$ replaces each $\lambda_{i,j}$ with $-\lambda_{i,j}$, yet leaves entries of $\Sigma$ unchanged. Thus, the above series in the $\lambda_{i,j}$ may only have even terms and
\begin{equation*}
    \frac{\int_R (Q_i(x)-\bar{Q}_i)^2S(x)dx}{\int_R S(x)dx}=\frac{2}{(d+2)^2}A_j+O(\Lambda^4).
\end{equation*}

Lastly, summing over $i$ gives \eqref{eq:sum lowest D-d}.

\section*{Acknowledgments}
We would like to thank Yariv Aizenbud for sharing code used to generate the airplane photos for our experiments, Sam Weissman for spotting a miscalculation in an earlier draft, and the reviewers for their helpful comments.

\bibliographystyle{siamplain}
\bibliography{BIB}

\end{document}